\documentclass[11pt,a4paper]{article}
\usepackage[utf8]{inputenc}
\usepackage[english]{babel}
\usepackage{amsmath}
\usepackage{comment}
\usepackage{amsthm}

\newcommand{\inst}[1]{}
\usepackage{latexsym}
\usepackage{amssymb,bbm,color}
\usepackage{color,soul}
\usepackage{enumitem}
\usepackage{tikz,pgfplots}
\usetikzlibrary{calc,decorations.pathreplacing}
\usepackage{ctable}
\usepackage{bm}
\usepackage{graphicx}
\usepackage{cellspace}
\allowdisplaybreaks[1]
\usepackage{hyperref}

\newtheorem{theorem}{Theorem}
\newtheorem{lemma}{Lemma}

\title{On a spatial-Temporal Decomposition of Optical Flow}

\author{Aniello Raffaele Patrone and Otmar Scherzer}

\date{}

\let\RE\Re
\let\Re=\undefined
\DeclareMathOperator{\Re}{\RE e}
\let\IM\Im
\let\Im=\undefined
\DeclareMathOperator{\Im}{\IM m}

\newcommand{\R}{\mathbbm R}

\newcommand{\N}{\mathbbm N}

\renewcommand{\d}{\mathrm d}

\newcommand{\abs}[1]{\left| #1 \right|}
\newcommand{\set}[1]{\left\{ #1 \right\}}
\newcommand{\norm}[1]{\left\| #1 \right\|}
\newcommand{\inner}[2]{\langle #1, #2 \rangle}
\newcommand{\vect}[1]{\left( \begin{array}{c} #1 \end{array} \right)}

\newcommand{\vx}{\vec{x}}
\newcommand{\vu}{\vec{u}}

\newcommand{\vw}{\vec{w}}

\newcommand{\vi}[1]{\vu^{(#1)}}
\newcommand{\vie}[1]{u^{(#1)}}
\newcommand{\Vi}[1]{\widehat{u}^{(#1)}}
\newcommand{\hi}[1]{\vec{h}^{(#1)}}
\newcommand{\vci}[2]{u_{#2}^{(#1)}}
\newcommand{\hci}[2]{h_{#2}^{(#1)}}
\newcommand{\vgci}[2]{\widehat{\widehat{u}}_{#2}^{(#1)}\!\!}
\newcommand{\vgcii}[2]{\widehat{u}_{#2}^{(#1)}\!}

\newcommand{\Reg}[1]{\mathcal{R}^{(#1)}}

\theoremstyle{definition} 
\newtheorem{example}{Example}

\begin{document}
\maketitle

\centerline{\scshape Aniello Raffaele Patrone$^*$}
\medskip
{\footnotesize
 \centerline{Computational Science Center}
   \centerline{University of Vienna}
   \centerline{ Oskar-Morgenstern Platz 1, 1090 Vienna, Austria}} 

\medskip

\centerline{\scshape Otmar Scherzer}
\medskip
{\footnotesize
 \centerline{ Computational Science Center}
   \centerline{University of Vienna}
   \centerline{Oskar-Morgenstern Platz 1, 1090 Vienna, Austria}
\centerline{and}
\centerline{Johann Radon Institute for Computational and Applied Mathematics }\centerline{(RICAM)}
\centerline{Altenbergerstra{\ss}e 69, 4040 Linz, Austria}
}

\bigskip


\begin{abstract}
In this paper we present a decomposition algorithm for computation 
of the spatial-temporal optical flow of a dynamic image sequence. We consider several applications, such as the 
extraction of temporal motion features and motion detection in dynamic 
sequences under varying illumination conditions, such as they appear for instance in psychological 
flickering experiments. 
For the numerical implementation we are solving an {\bf integro-differential} equation by a fixed 
point iteration. For comparison purposes we use a standard time dependent optical flow algorithm, 
which in contrast to our method, constitutes in solving a spatial-temporal {\bf differential} equation.
\end{abstract}

\section{Introduction}
Analyzing the motion in a dynamic sequence is of interest in many fields of applications, 
like human computer interaction, medical imaging, psychology, to mention but a few.
In this paper we study the extraction of motion in dynamic sequences by means of the optical flow, which 
is the apparent motion of objects, surfaces, and edges in a dynamic visual scene caused by the relative 
motion between an observer and the scene.
There have been proposed several computational approaches for optical 
flow computations in the literature. In this paper we emphasize on variational methods. 
In this research area the first method is due to Horn \& Schunck \cite{HorSchu81}. 
Like many alternatives and generalizations, the Horn \& Schunck method calculates the flow 
from {\bf two consecutive} frames.
Here, we are calculating the optical flow from {\bf all} available frames simultaneously. 
Spatial-temporal optical flow methods were previously studied by Weickert \& Schn\"orr
\cite{WeiSchn01a,WeiSchn01b}, \cite{BorItoKun03}, \cite{WanFanWan08} 
and \cite{AndSchZul15}, to name but a few. However, in contrast to these paper we emphasize 
on a {\bf decomposition} of the optical flow into components, instead of calculating the flow itself. 

Variational modeling of patterns in {\bf stationary} images has been initialized with the seminal book of Y.~Meyer \cite{Mey01}. 
In the context of total variation regularization, reconstructions of patterns was studied first in \cite{VesOsh03}. 
Here, we are implementing similar ideas as have been used before for variational image denoising 
\cite{AubAuj05,AujCha05,AujKan06,AujAubBlaCha05,AujGilChaOsh06,DuvAujVes10,VesOsh04} and 
optical flow decomposition \cite{AbhBelSch09,KohMemSchn03,YuaSchnSte07,YuaSteSchn08,YuaSchnSte09}. 
However, a conceptual difference is that we aim for extracting {\bf temporal} features, while, 
in all the cited papers the decomposition was for finding spatial components of the flow. 
We emphasize that the proposed method is one of very few variational optical flow algorithms in a 
space-time regime and within this class, this algorithm is the only spatial-temporal {\bf decomposition} algorithm. 

The outline of this paper is as follows:
In Section \ref{sec:registration} we review the optical flow equation. In Section \ref{sec:anEx} we present analytical examples of the 
optical flow equation in case of illumination changes. In Section \ref{section:oflow} we introduce the 
new model for spatial-temporal optical flow decomposition. We formulate it as a minimization problem 
and derive the optimality conditions for a minimizer. 
In Section \ref{sec:implementation} we derive a fixed point algorithm for 
numerical minimization of the energy functional. 
Finally in Section \ref{sec:experiments} and Section \ref{sec:conclusion} we present
experiments, results and a discussion of them.

\section{Registration and optical flow}
\label{sec:registration}
The problem of aligning dynamic sequences $f(\cdot,t)$, $t \in (0,1)$ can be formulated as the operator equation for 
finding a flow $\Psi$ of diffeomorphisms, 
\begin{equation*}
\Psi(\cdot,t) : \Omega \to \Omega, \quad \forall\, t \in (0,1),
\end{equation*}
such that 
\begin{equation}\label{eq:registration}
 f(\Psi(\vx,t),t) = f(\vx,0),\qquad \forall\, \vx \in \Omega \text{ and } t \in (0,1) .
\end{equation}
For the sake of simplicity of presentation we consider the time interval as the unit interval all along the paper.

Differentiation of \eqref{eq:registration} with respect to $t$ for a fixed $\vx$ gives
\begin{equation}\label{eq:registration_oflow}
\nabla f (\Psi(\vx,t),t) \cdot \partial_t \Psi(\vx,t) + \partial_t f(\Psi(\vx,t),t)= 0,\qquad \forall \,
\vx \in \Omega \text{ and }  t \in (0,1) .
\end{equation}

Switching from a Lagrangian to an Eulerian description we obtain the 
{\bf optical flow equation (OFE)} on $\Omega$:
\begin{equation}\label{eq:OFE}
\nabla f(\vx,t) \cdot \vu(\vx,t) + \partial_t f(\vx,t) =0,\qquad \forall\, \vx \in \Omega \text{ and } t \in (0,1).
\end{equation}
Although the derivation is based on a constant brightness assumption along characteristics, mathematically, 
\eqref{eq:OFE} even makes sense under varying illumination conditions. However, as we show in simple examples below, 
standard regularity assumptions on the optical flow are violated when the characteristics degenerate (collapse or originate) or 
when the illumination changes.
Instead of solving \eqref{eq:OFE} usually the relaxed problem is considered, which consists in minimizing the functional 
\begin{equation}\label{OFE_rel}
\text{argmin } S(\vu):=\int_\Omega (\nabla f(\vx,t) \cdot \vu(\vx,t) + \partial_t f(\vx,t))^2\,\d\vx\,,\qquad \forall t \in (0,1),
\end{equation}
subject to appropriate constraints.

\section{The optical flow equation in case of illumination changes}
\label{sec:anEx}

In this section we are providing simple motivating examples explaining properties of the solution of the 
optical flow equation \eqref{eq:OFE} under changing illumination conditions. Here we are restricting attention 
to a spatial domain $\emptyset \neq \Omega=(0,1) \subset \R$.

The following two examples simulate a day to night illumination situation. The optical flow is calculated 
analytically, and the level lines of $f$ are visualized, which represent the trajectories $\Psi$ of constant 
brightness. As we show, smoothness of the optical flow is affected by changing illumination and 
in the first example also by joining of characteristic curves.

\begin{example} \label{ex:flickering}

In this example the flow is not even an element of the Bochner-space $L^2((0,1);H^{-1}(\Omega))$, meaning that the anti-derivative 
of $u$ with respect to time is not in $L^2((0,1);L^2(\Omega))$. Because $L^2((0,1);H^{-1}(\Omega))$ is a strict superset 
of $L^2((0,1);L^2(\Omega))$, the elements have in general less regularity (smoothness). The next Example \ref{ex:fl:2}
provides a flow where $f$ models again changing illumination. Here the flow is in $L^2((0,1);H^{-1}(\Omega))$ but not in 
$L^2((0,1);L^2(\Omega))$. We conjecture from the difference of the two examples that the difference in smoothness is caused 
by the fact that in the first example two characteristics are joining during 
the evolution. 

We consider the 1D optical flow equation, to solve for $u$ in 
\begin{equation}
 \label{eq:ofe_1}
 \partial_x f(x,t) u(x,t) +\partial_t f(x,t)=0, \quad \forall (x,t) \in (0,1) \times (0,1)
\end{equation}
for the specific data
\begin{equation}
\label{eq:specific}
f(x,t)=\tilde{f}(x) g(t), \quad \forall (x,t) \in (0,1) \times (0,1) .
\end{equation}
$f$ represents a static scene $\tilde{f}$, which is affected by brightness variations over time, described by $g$. 
We are more specific and take:
\begin{equation}
\label{eq:10a}
 \tilde{f}(x) = x(1-x) \text{ and } g(t)=1-t, \quad \forall (x,t) \in (0,1) \times (0,1)  .
\end{equation}
The function $f$ and the level lines are plotted in Figure \ref{fig:hatF} and 
the optical flow can be calculated explicitly: 
\begin{equation*}
 u(x,t)=\frac{x(1-x)}{1-2x} \frac{1}{1-t}, \quad \forall (x,t) \in (0,1) \times (0,1), 
\end{equation*}
indicates a transport of intensities from outside to the center $1/2$.
We observe that $u(1/2,t)$ and $u(x,1)$ are singularities of the optical flow.
\begin{figure}[htp]
\begin{center}
\includegraphics[width=0.45\textwidth]{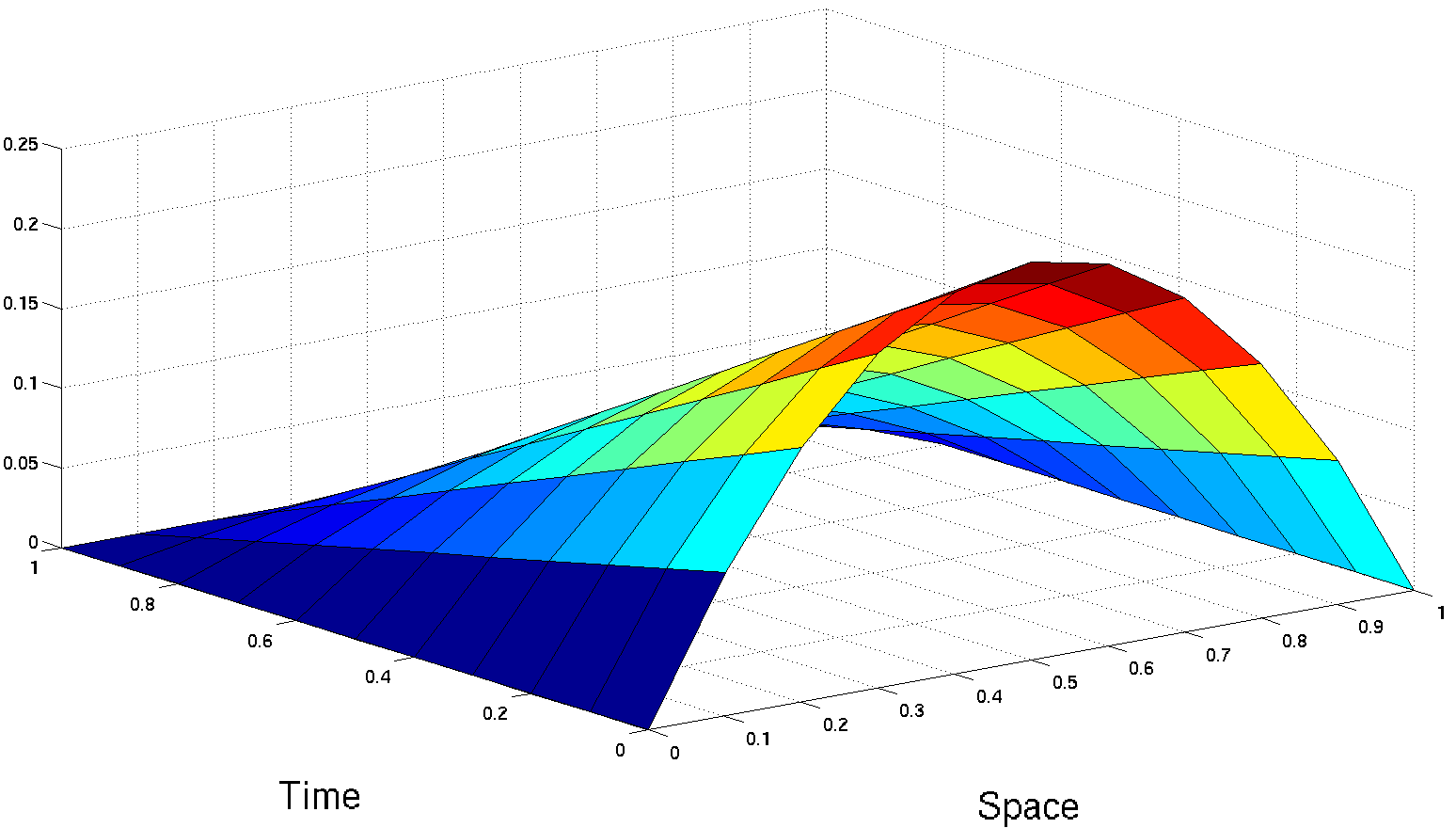} \qquad \includegraphics[width=0.45\textwidth]{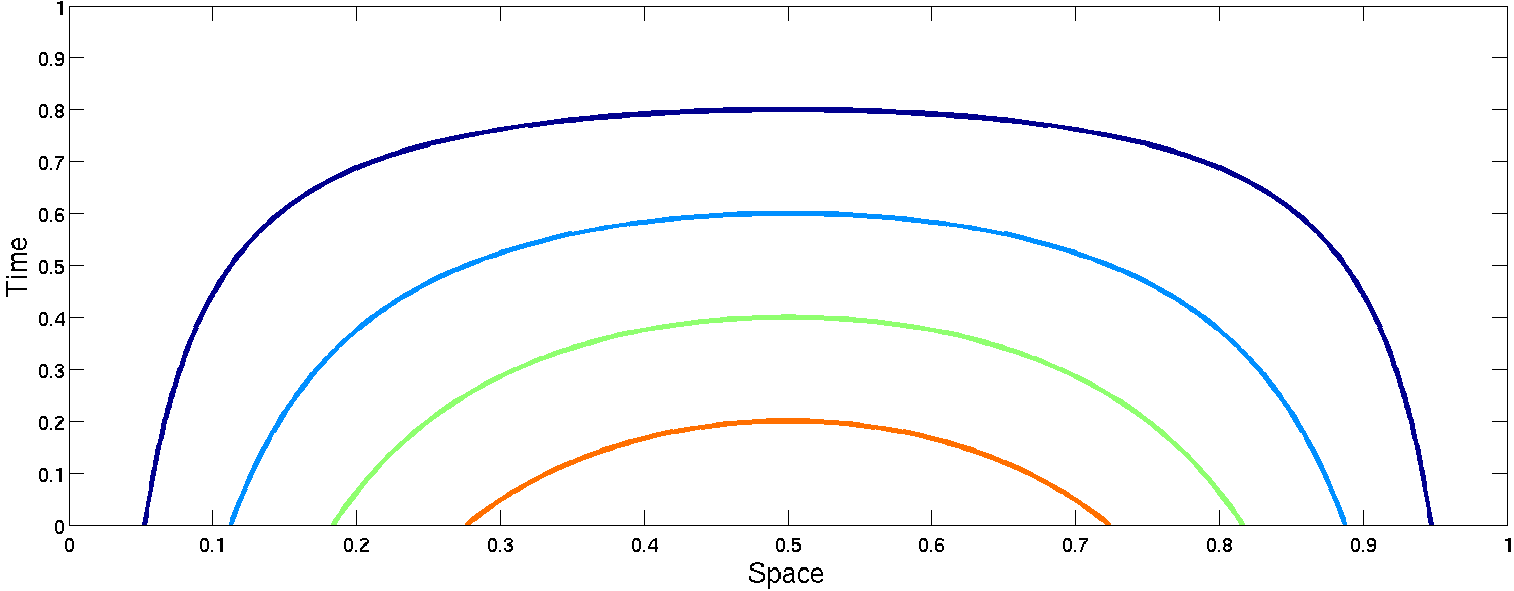}
\caption{$f(x,t)=x(1-x)(1-t)$ from \eqref{eq:10a}. Level lines of $f$ are parametrized by $(\Psi(x,t),t)$.}
\label{fig:hatF}
\end{center}
\end{figure}
From the definition of $u$ it follows that
\begin{equation*}
 \hat{u}(x,t):=\int_0^t u(x,\tau)\,\d \tau = - \frac{x(1-x)}{1-2x} \log(1-t), \quad \forall (x,t) \in (0,1) \times (0,1) 
\end{equation*}
and thus
\begin{equation*}
\norm{\hat{u}}_{L^2((0,1)^2)}^2
 = 
 \int_0^1 \frac{x^2(1-x)^2}{(1-2x)^2} \,\d x \int_0^1 \log^2(1-t) \,\d t\\
 = 
 2 \int_0^1 \frac{x^2(1-x)^2}{(1-2x)^2}\,\d x 
 = \infty,
\end{equation*}
or in other words $\hat{u} \notin L^2((0,1)^2)$. 
\end{example}

\begin{example}
\label{ex:fl:2}This example is similar to Example \ref{ex:flickering}, and simulates again a day to night illumination, 
with the difference that characteristics of $f$ never join. 
We consider input data $f$ of the form \eqref{eq:specific} with 
\begin{equation}
\label{eq:10b}
 \tilde{f}(x) = x(1-x) \text{ and } g(t)=\exp \set{-\frac{1}{\beta}(1-t)^{\beta}} \text{ with some } 0 < \beta < 1
\end{equation}
for $(x,t) \in \hat{\Omega}:=(0,1/4) \times (0,1)$. 
The optical flow is given by 
\begin{equation*}
 u(x,t)=-\frac{x(1-x)}{1-2x} (1-t)^{\beta-1}.
\end{equation*}
Integrating this function over time gives
\begin{equation*}
 \hat{u}(x,t):=\int_0^t u(x,\tau)\,\d \tau = \frac{x(1-x)}{1-2x} \frac{1}{\beta}((1-t)^\beta-1),
\end{equation*}
\begin{figure}[h]
\begin{center}
\includegraphics[scale=0.6]{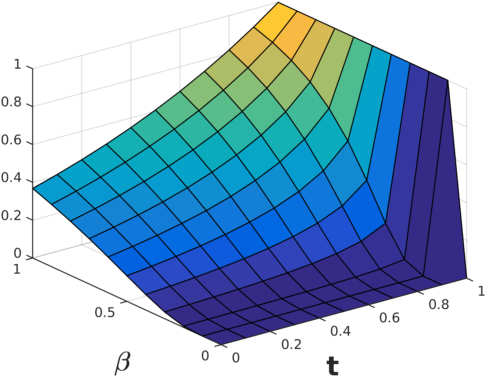} 
\end{center}
\caption{$g(t)=\exp \set{-\frac{1}{\beta}(1-t)^\beta}$}
\end{figure}
and consequently with

\begin{equation*}
 C:= \int_0^{1/4} \frac{x^2(1-x)^2}{(1-2x)^2}\,\d x < \infty\,,
\end{equation*}
we get
\begin{equation*}
\begin{aligned}
\norm{u}_{L^2(\hat{\Omega})}^2
&= C \int_0^1  t^{2\beta-2}\,\d t = \left\{ \begin{array}{rl} 
                                           C \frac{1}{2\beta-1} & \text{ if } \beta > \frac{1}{2}\,,\\
                                           \infty & \text{else}\;.
                                          \end{array} \right.\\
\norm{\hat{u}}_{L^2(\hat{\Omega})}^2
&= \frac{C}{\beta^2} \int_0^1  t^{2\beta}-2t^\beta+1\,\d t =                                           
                                          \frac{C}{\beta^2} 
                                          \left( \frac{1}{2\beta+1} - \frac{2}{\beta+1} + 1\right)& \text{ if } \beta > 0 \,.
\end{aligned}
\end{equation*}

This shows that $u \notin L^2(\hat{\Omega})$ for every $\beta\in(0,\frac{1}{2}]$, 
but  $\hat{u} \in L^2(\hat{\Omega})$ for all $\beta\in(0,1)$.
\end{example}
The bottom line of these examples is that illumination changes, such as 
flickering, may result in singularities of the optical flow
and a violation of standard smoothness assumptions of the optical flow (such as $\vec{u} \in L^2((0,1); H^s(\Omega))$ for some $s>0$).
The potential appearance of the singularities motivates us to consider 
regularization terms for optical flow computations, which allow for 
singularities over time, such as negative Sobolev-norms.

\section{Optical flow decomposition: basic setup and formalism} 
\label{section:oflow}

In this paper we derive an optical flow model which allows for decomposing the flow field into 
spatial and temporal components. We consider each frame of the movie $\set{f(\cdot,t): t \in (0,1)}$ 
defined on the two-dimensional spatial domain $\Omega=(0,1)^2$.

We assume that the optical flow field is a compound of two flow field components
\begin{equation*}
\begin{aligned}
\vu(\vx,t) =&
\vi{1}(\vx,t) + \vi{2}(\vx,t)= \vect{\vci{1}{1}(\vx,t) \\ \vci{1}{2}(\vx,t)} + 
\vect{\vci{2}{1}(\vx,t) \\ \vci{2}{2}(\vx,t)},\\ 
& \forall \vec{x}\in \Omega \text{ and } t \in (0,1).
\end{aligned}
\end{equation*}
Because there appears a series of indices and variables we specify the notation in Table \ref{tab:continuous}.
\begin{table}[bht]
\begin{center}
 \begin{tabular}{|S r|p{0.6\textwidth}|}
  \hline
 $\vx = (x_1,x_2)$ & vector in two-dimensional Euclidean space \\
  \hline
 $\partial_k = \frac{\partial}{\partial x_k}$ & differentiation with respect to spatial variable $x_k$ \\
  \hline
 $\partial_t = \frac{\partial}{\partial t}$ & differentiation with respect to time \\
  \hline
 $\nabla = (\partial_1, \partial_2)^T$ & gradient in space \\
  \hline
 $\nabla_3 = (\partial_1, \partial_2, \partial_t)^T$ & gradient in space and time \\ 
  \hline
 $\nabla \cdot = \partial_1 + \partial_2$ & divergence in space \\ 
  \hline
 $\nabla_3 \cdot = \partial_1 + \partial_2 + \partial_t$ & divergence in space and time\\
  \hline
 $\vec{n}$ & outward pointing normal vector to $\Omega$\\ 
  \hline 
 $f$ & input sequence\\
 \hline
 $f(\cdot,t)$ & movie frame\\
 \hline
 $\vi{i}$ & optical flow module, $i=1,2$\\ 
 \hline
 $\vu = \vi{1} + \vi{2}$ & optical flow \\
 \hline
 $\vci{i}{j}$ & $j$-th optical flow component of the $i$-th module \\
 \hline
 $\widehat{u}(\cdot,t) = \int_0^t u(\cdot,\tau)\,\d \tau$ & primitive of $u$\\
 \hline
  $\widehat{\widehat{u}}(\cdot,t) = -
  \int_t^1 \widehat{u}(\cdot,\tau)\,\d \tau$ & 2nd primitive of $u$ - note that $\partial_t \widehat{\widehat{u}}(\cdot,t)=\widehat{u}(\cdot,t)$\\
 \hline
 \end{tabular}
  \vspace{0.3cm}
\caption{Continuous notation.}
\label{tab:continuous}
\end{center}
\end{table}
\bigskip\par\noindent
The OFE-equation \eqref{eq:OFE} contains four unknown (real valued) functions $u_j^{(i)}$, $i,j=1,2$, and thus 
is highly under-determined. To overcome the under-determinacy, the problem is formulated as a 
constrained optimization problem, to determine, for some fixed $\alpha > 0$,
\begin{equation}
 \label{eq:hard}
 \text{argmin} \left( \Reg{1}(\vi{1}) + \alpha \Reg{2}(\vi{2})\right)
\end{equation}
subject to \eqref{eq:OFE}. 
\bigskip\par\noindent
Instead of solving the constrained optimization problem, we use a soft variant and minimize the unconstrained regularization functional:
\begin{equation}
\begin{aligned}
\label{eq:E}
\mathcal{F}(\vi{1},\vi{2}) &:= \mathcal{E}(\vi{1},\vi{2}) + \sum_{i=1}^2 \alpha^{(i)} \Reg{i}(\vi{i}),\\
\mathcal{E}(\vi{1},\vi{2}) &:= 
 \int\limits_{\Omega \times(0,1)}\!\!\!\!\!\!(\nabla f\cdot (\vi{1}+\vi{2}) + \partial_t f)^2\,
\d \vx\d t \quad \text{ with } \alpha = \frac{\alpha^{(2)}}{\alpha^{(1)}}.
\end{aligned}
\end{equation}
For the sake of simplicity of presentation, we omit arguments of the functions $\vci{i}{j}$ and $f$, whenever 
it simplifies the formulas without causing misinterpretations.
\bigskip\par\noindent
In the following we design the regularizers $\Reg{i}$:
\begin{itemize}
\item For $\Reg{1}$ we use a common spatial-temporal regularization functional for optical flow regularization (see for instance 
\cite{WeiSchn01b}): 
\begin{equation}\label{ConstraintU1}
 \Reg{1}(\vi{1})
 := \int\limits_{\Omega \times(0,1)}\nu \left(\abs{\nabla_3 \vci{1}{1}}^2 +\abs{\nabla_3 \vci{1}{2}}^2\right)\d \vx\d t,
\end{equation}
where $\nu: [0,\infty) \rightarrow [0,\infty)$ is a monotonically increasing, differentiable function satisfying that 
$r \to \nu(r^2)$ is convex.

According to \cite{WeiSchn01b} we use
\begin{equation}
\nu(r)=\epsilon r + (1-\epsilon)\lambda^2 \left(\sqrt{1+\frac{r}{\lambda^2}}-1\right)\,, \quad \forall r \in [0,\infty)\,,
\end{equation}
with some fixed $0<\epsilon \ll 1$ and $\lambda >0$. Note, that in \cite{WeiSchn01b} the term $-1$ is not used. 
However, since it is a constant, it does not influence the optimization. Using this term guarantees that $\nu(0)=0$.
Moreover, we denote by $\nu'$ the derivative of $\nu$ with respect to $r$.
\item 
$\Reg{2}$ is designed to penalize for variations of the second component in time. 
Motivated by Y.~Meyer's book \cite{Mey01}, we introduce a regularization term, which is non-local in {\bf time}. 
We have seen in Examples \ref{ex:flickering} and \ref{ex:fl:2} that $u$ may violate $L^2$-smoothness in case of 
changing illumination conditions.
Variations of Meyer's $G$-norm where used in energy functionals for calculating {\bf spatial} decompositions 
of the optical flow \cite{AbhBelSch09,KirLanSch14}. It is a challenge to compute the $G$-norm efficiently, and therefore workarounds have been proposed. 
For instance \cite{VesOsh03} proposed as an alternative to the $G$-norm the following realization for the $H^{-1}$--norm:
For a generalized function $u: (0,1) \to \R$, they defined  
$\norm{u}^2_{H^{-1}} := - \int_0^1 u(t) \partial_{tt}^{-1} u(t) \d t.$ 
Here, we use the following {\bf temporal} $H^{-1}$-norm for regularization:
\begin{equation}
\label{eq:2}
 \Reg{2}(\vi{2}) := \!\!\!\!\int\limits_{\Omega \times (0,1)} \!\!\!\! \abs{ \int_0^t \vi{2}(\vx,\tau) \d \tau }^2 \!\!\!\!\d \vx \d t =
\sum_{j=1}^2\int\limits_{\Omega \times (0,1)} \!\!\!\! \left(\vgcii{2}{j} \right)^2 \!\! \d \vx \d t .
\end{equation}
To see the analogy with the $\norm{\cdot}_{H^{-1}}$-norm from \cite{VesOsh03} 
we note that the second primitive of the optical flow component $\vi{2}$, as defined in Table \eqref{tab:continuous}, satisfies 
\begin{equation}
\label{eq:pde2a}
 \begin{aligned}
  \vgci{2}{j}(\vx,1)=0 ,\quad
  \partial_t \vgci{2}{j}(\vx,0) = \vgcii{2}{j}(\vx,0) = 0,\quad \forall j=1,2 \text{ and }\vx \in \Omega.
 \end{aligned} 
 \end{equation}
Then, by integration by parts it follows that
\begin{equation*}
- \int_0^1 \underbrace{\vgci{2}{j}(t)}_{=\partial_{tt}^{-1}\vci{2}{j}} \vci{2}{j}(t)\,\d t = 
\int_0^1 \left(\vgcii{2}{j}(t)\right)^2 \,\d t
\end{equation*}
and therefore 
\begin{equation}
\label{eq:2a}
\Reg{2}(\vi{2}) = \sum_{j=1}^2 \int_\Omega \norm{\vci{2}{j}(\vx,\cdot)}_{H^{-1}}^2\d \vx .
\end{equation}
\end{itemize}
Existence of a minimizer of $\mathcal{F}$, defined in \eqref{eq:hard}, in an infinite dimensional 
functional analytical setting is a complicated issue. However, for some surrogate model, we can guarantee 
well--posedness by using the following lemma from \cite{AbhBelSch09}:
\begin{lemma}[\cite{AbhBelSch09}]
Let $f \in C^1(\overline{\Omega} \times [0,1])$. We consider $t \in [0,1]$ fixed, and for  
$f := f(\cdot,t)$ we define  $A_0:=  \nabla f (\nabla f)^T$. Then 
$\inner{\vw_1}{\vw_2}_{A_0} = \int_\Omega \vw_1^T A_0 \vw_2$ is an inner product and 
$\abs{\vw}_{A_0}^2 := \inner{\vw}{\vw}_{A_0} = \int_\Omega \vw^T A_0 \vw$ is a seminorm on $L^2(\Omega;\R^2)$.
\end{lemma}

Let 
\begin{equation}
\label{eq:G}
\vec{\rho}:=\frac{1}{\abs{\nabla f}}(-f_tf_x,-f_tf_y)^T. 
\end{equation}
Then $A_0^{1/2}=\frac{1}{|\nabla{f}|}A_0$, where the root of a matrix is defined via spectral decomposition. 
Moreover,
\begin{equation}
\label{eq:equiv}
  \norm{A_0^{1/2} \vw - \vec{\rho}}^2_{L^2(\Omega;\R^2)}
  = \norm{\nabla f\cdot \vw + f_t}^2_{L^2(\Omega)} , \quad \forall\vw \in L^2(\Omega;\R^2).
\end{equation}

Let 
\begin{equation}
 \label{eq:A}
 A:=((A_0)^T A_0+\epsilon \text{Id})^\frac{1}{2},
\end{equation}
where, $\text{Id} \in \R^{2 \times 2}$ denotes the identity matrix and $\epsilon>0$ is a small regularization parameter. 
$A$ is uniformly positive definite on $\overline{\Omega}$. We note that in comparison with \cite{AbhBelSch09}, we are 
not using a smoothed 
version of $A$, because we already made the assumption that $A \in C^0(\overline{\Omega} \times [0,1];\R^{2 \times 2})$ (because $f \in C^1(\overline{\Omega} \times [0,1])$).
Using the notation
\begin{equation}
 \label{eq:TG}
\vec{\phi}=A^{-\frac{1}{2}}\vec{\rho}\,,
\end{equation}
we assume
\begin{equation}
\label{eq:equiv2}
\norm{\nabla f\cdot \vw + f_t}_{L^2(\Omega)} \approx
\norm{\vw-\vec{\phi}}_A\;.
\end{equation}
After choosing a fixed $\epsilon > 0$ we consider minimization of the surrogate model, consisting in minimizing the functional
\begin{equation}
\begin{aligned}
\label{eq:E_sur}
\mathcal{F}_s(\vi{1},\vi{2}) &:= \mathcal{E}_s(\vi{1},\vi{2}) + \sum_{i=1}^2 \alpha^{(i)} \Reg{i}(\vi{i}),\\
\mathcal{E}_s(\vi{1},\vi{2}) &:= 
 \int\limits_{\Omega \times(0,1)} \norm{(\vi{1}+\vi{2})-\vec{\phi}}_A^2
\d \vx\d t \quad \text{ with }\alpha = \frac{\alpha^{(2)}}{\alpha^{(1)}}
\end{aligned}
\end{equation}
over the Bochner-Space 
\begin{equation*}
X := \set{\vi{1} \in W^{1,2}(\Omega \times (0,1);\R^2): \vi{1}(\cdot,T)=0} \times L^2(\Omega \times (0,1);\R^2).
\end{equation*}
We are proving that the surrogate model attains a minimizer on $X$. Note that $W^{1,2}(\Omega \times (0,1);\R^2)$ 
is the space of vector valued, weakly differentiable functions with respect to space and time, and also recall that 
$L^2((0,1);L^2(\Omega;\R^2))=L^2((0,1)\times \Omega;\R^2)$.
\begin{theorem}
$\mathcal{F}_s$ attains a minimizer on $X$.
\end{theorem}
\begin{proof} 
The proof is done by several estimates: 
\begin{itemize}
 \item Because $(\vec{0},\vec{0}) \in X$ it follows that 
       $C:=\mathcal{F}_s (\vec{0},\vec{0}) = \int_0^1 \norm{\vec{\phi}}_A^2\,\d t < \infty.$
 \item Now, let $(\vi{1}_n,\vi{2}_n)$ be a minimizing sequence of $\mathcal{F}_s$, then
       \begin{equation*}
        \begin{aligned}
         ~ & \int_0^1 \norm{(\vi{1}_n + \vi{2}_n)-\vec{\phi}}_A^2 \d t \leq C,\\
         ~ & \int_0^1 \norm{\nabla_3 \vi{1}_n}_{L^2(\Omega;\R^{3 \times 2})}^2\d t \leq 
         \frac{1}{\epsilon} \Reg{1}(\vi{1}_n) \leq \frac{C}{\epsilon \alpha^{(1)}}, \\
         ~ & 
         \norm{\widehat{\vec{u}}_n^{(2)}}_{L^2(\Omega \times (0,1);\R^2)} \leq \frac{C}{\alpha^{(2)}}.
        \end{aligned}
       \end{equation*}
 \item For the surrogate model $\norm{\cdot}_A$ is a norm, which is equivalent to the $L^2(\Omega;\R^2)$-norm, and 
       thus $(\vi{1}_n + \vi{2}_n)$ is uniformly bounded in $L^2((0,1);L^2(\Omega))$, and therefore admits a weakly 
       convergent subsequence with limit $\vec{\phi} \in L^2((0,1);L^2(\Omega))$. 
       This subsequence, in turn has a subsequence (which is also denoted by $\cdot_n$, for the sake of simplicity of notation) 
       such that also $(\partial_t \vi{1}_n)$ and $(\widehat{\vec{u}}_n^{(2)})$ are weakly convergent in 
       $L^2(\Omega \times (0,1);\R^2)$, to some $\vec{\mu}$ and $\vec{\psi}$, respectively.   
 \item Let $\vec{\zeta} \in L^2(\Omega \times (0,1);\R^2)$ be arbitrary, then $(\vx,t) \to \int_t^1 \vec{\zeta}(\vx,\tau) \d \tau 
       \in L^2(\Omega \times (0,1);\R^2)$. Using the weak convergence of the subsequence it follows that 
       \begin{equation}
        \begin{aligned}
         \int_{\Omega \times (0,1)} \widehat{\vec{\phi}} \cdot \vec{\zeta}  \, \d \vx \d t
         = & \int_{\Omega \times (0,1)} \vec{\phi}(\vx,t) \cdot \int_t^1 \vec{\zeta}(\vx,\tau) \d \tau \d \vx \d t \\
         = & \lim_{n \to \infty} \int_{\Omega \times (0,1)} (\vi{1}_n + \vi{2}_n)(\vx,t) \cdot \int_t^1 \vec{\zeta}(\vx,\tau) \d \tau \d \vx \d t \\
         = & \lim_{n \to \infty} \int_{\Omega \times (0,1)} \int_0^t (\vi{1}_n + \vi{2}_n)(\vx,\tau)\d \tau \cdot \vec{\zeta}(\vx,t) \d t \d \vx \\
         = & \lim_{n \to \infty} \int_{\Omega \times (0,1)} \widehat{\vec{u}}_n^{(1)} \cdot \vec{\zeta} \; \d t \d \vx + 
               \int_{\Omega \times (0,1)} \vec{\psi} \cdot  \vec{\zeta} \; \d t \d \vx\;.
        \end{aligned}
       \end{equation}
       This means that $(\widehat{\vi{1}_n})$ converges weakly to $\widehat{\vec{\phi}}-\vec{\psi}$, and in particular 
       $\norm{\widehat{\vi{1}_n}}_{L^2(\Omega \times (0,1))}$ is uniformly bounded, let us say by $D$.
       The last item shows that that $(\vi{1}_n)$ converges to $\vec{\phi}-\partial_t \vec{\psi}$ in a distributional sense.
 \item From integration by parts it follows that if $\vi{1}_n(\cdot,T) \equiv 0$ that
        \begin{equation}
        \begin{aligned}
         ~ & \int_{\Omega \times (0,1)} \abs{\vi{1}_n}^2 \d \vx \d t \\
          = & - \int_{\Omega \times (0,1)} \partial_t \vi{1}_n(\vx,t) \cdot \left( \int_0^t \vi{1}_n(\vx,\tau) \,\d \tau\right) \d t \d \vx\\
          \leq & 
          \sqrt{
          \int_{\Omega \times (0,1)} \abs{\partial_t \vi{1}_n(\vx,t)}^2 \d t \d \vx 
          \int_{\Omega \times (0,1)} \abs{\int_0^t \vi{1}_n(\vx,\tau) \,\d \tau}^2 \d t \d \vx} \leq 
          D \sqrt{\frac{C}{\epsilon \alpha^{(1)}}}\;.
        \end{aligned}
       \end{equation}
       As a consequence we have $\vec{\phi}-\partial_t \vec{\psi} \in L^2(\Omega \times (0,1);\R^2).$
\item The final results follows from the lower semicontinuity of the functionals 
      $\vec{u} \to \int_0^1 \norm{\vec{u}-\vec{\phi}}^2_A\,dt$ and $\Reg{i}(\vi{i})$ for $i=1,2$.
\end{itemize}
\end{proof}
We emphasize that the critical issue in the above proof is the coercivity, which could be enforced 
by various other modifications of the functional $\mathcal{F}$ as well. 
Another possibility, to the one mentioned above (which we feel to be less elegant) is to add a small 
regularization term for the $L^2$-norm of $\vi{1}$ to the functional $\mathcal{F}$.
Since the use of surrogate models has only marginal impact on the numerical reconstructions we ignore their use in the following.

\subsection{Energy functional and minimization}
We are determining the optimality conditions for minimizers of $\mathcal{F}$ introduced in \eqref{eq:E}.
Necessary conditions for a minimizer are that the directional derivatives vanish for all 2-dimensional 
vector-valued functions $\hi{j}:\Omega \times (0,1) \to \R^2$, $j=1,2$. 
To formulate these conditions we use the simplifying notation:
\begin{equation*}
 (\mathcal{E},\mathcal{F}):=(\mathcal{E},\mathcal{F})(\vi{1},\vi{2}), \; \Reg{i} := \Reg{i}(\vi{i}) \text{ and }
 \text{res} = \nabla f \cdot (\vi{1}+\vi{2}) + \partial_t f\;.
\end{equation*}
Therefore the directional derivative of $\mathcal{F}$ in direction $\hi{j}$ at $\vec{u}=(\vi{1},\vi{2})$ is given by:
 \begin{equation*} 
(\partial_{\hi{j}}\mathcal{F}) \vec{u}
=\lim_{s \to 0} 
\frac{\mathcal{F}(\vi{1}+s\delta_{1j}\hi{1},\vi{2}+s\delta_{2j}\hi{2})-\mathcal{F}}{s}=0
\end{equation*}
where $\delta_{ij} = 1$ for $i=j$ and zero else is the Kronecker symbol.
The gradient of the functional $\mathcal{F}$ from \eqref{eq:E} can be determined by calculating the 
directional derivatives of $\mathcal{E}$ and $\Reg{i}$, separately. 
\begin{itemize}
\item
The directional derivative of $\mathcal{E}$ in direction $\hi{j}$ is given by 
\begin{equation}
\label{eq:r3}
   (\partial_{\hi{j}} \mathcal{E}) \vec{u}
 =  2\int\limits_{\Omega \times (0,1)} \text{res} (\nabla f \cdot \hi{j})\,\d \vx \d t .
\end{equation}

\item The directional derivative of $\Reg{1}$ at $\vi{1}$ in direction $\hi{1}$ is determined as follows:
Let us abbreviate for simplicity of presentation 
\begin{equation*}
\nu :=\nu \left(\abs{\nabla_3 \vci{1}{1}}^2 + \abs{\nabla_3 \vci{1}{2}}^2 \right), \quad 
\nu' :=\nu' \left(\abs{\nabla_3 \vci{1}{1}}^2 + \abs{\nabla_3 \vci{1}{2}}^2 \right),
\end{equation*}
then the directional derivative of $\Reg{1}$ in direction $\hi{1}$ 
(which we assume to have compact support in $\Omega \times (0,1)$)
at $\vi{1}$ is given by
\begin{equation}
 \label{eq:r1}
\begin{aligned}
& (\partial_{\hi{1}} \Reg{1})\vi{1}
=  \lim_{s \to 0}\frac{\Reg{1}(\vi{1}+s\hi{1}) - \Reg{1}}{s} \\
= & -2\int_{\Omega\times (0,1)} \nabla_3 \cdot \left(\nu' \nabla_3 \vci{1}{1} \right)\hci{1}{1} + 
    \nabla_3 \cdot \left(\nu'\nabla_3 \vci{1}{2}\right)\hci{1}{2}\d \vx \d t,
\end{aligned}
\end{equation}
where integration by parts is used to prove the final identity.
\item 
The directional derivative of $\Reg{2}$ is derived as follows:
\begin{equation}
\label{eq:r2}
 \begin{aligned} 
  & (\partial_{\hi{2}} \Reg{2})\vi{2}
=  \lim_{s \to 0}
\frac{\Reg{2}(\vi{2}+s\hi{2}) - \Reg{2}}{s} 
=  2 \sum_{j=1}^2 \int_{\Omega\times (0,1)} \vgcii{2}{j} \widehat{h}_j^{(2)} \d \vx \d t .
 \end{aligned}
\end{equation}

\end{itemize}
Moreover, it follows by integration by parts of the last line of \eqref{eq:r2}  with respect to $t$ that 
\begin{equation}
\begin{aligned}
\label{eq:r2a}
(\partial_{\hi{2}} \Reg{2})\vi{2} = - 2 \sum_{j=1}^2 \int_{\Omega\times (0,1)} \vgci{2}{j}  \hci{2}{j} \d \vx \d t .
\end{aligned}
\end{equation}

Now, because of \eqref{eq:r1} and \eqref{eq:r3} it follows that the minimizer $\vi{i}$, $i=1,2$ has to satisfy for every $j=1,2$, 
\begin{equation}
\label{eq:optin}
\begin{aligned}
\partial_j f (\nabla f \cdot (\vi{1}+\vi{2}) + \partial_t f) - \alpha^{(1)} 
\nabla_3 \cdot \left(\nu' \nabla_3 \vci{1}{j}\right) &=  0\text{ in } \Omega \times (0,1),\\
 \partial_{\vec{n}} \vci{1}{j}&=0 \text{ in } \partial \Omega \times (0,1),\\
 \partial_t \vci{1}{j}&=0 \text{ in } \Omega \times \set{0,1} .
\end{aligned}
\end{equation}
Since equations  
\eqref{eq:r3} and \eqref{eq:r2a} hold for all $\hi{2}_{j}$, it follows that for every $j=1,2$, 
\begin{equation}
\begin{aligned}
\label{eq:pde2}
\partial_j f (\nabla f \cdot (\vi{1}+\vi{2}) + \partial_t f) -  \alpha^{(2)} \vgci{2}{j} = 0 \text{ in } \Omega \times (0,1) .
\end{aligned}
\end{equation}
Thus the optimality conditions for a minimizer are \eqref{eq:optin} and \eqref{eq:pde2}.
The solution of Equation \eqref{eq:optin} has to be understood in a weak sense. The optimality condition is formal and 
would be exact if $\mathcal{F}$ would be coercive. As we demonstrated above there are several possibilities of surrogate models,
which provide coercivity. The surrogate model outlined above leading to this optimality condition would require a Dirichlet 
boundary condition for $\vi{1}$ at $\Omega \times \set{1}$. 
If we would indeed complement $\mathcal{F}$ by $\alpha_3 \norm{\vi{1}}_{L^2(\Omega \times (0,1))}$, with small $\alpha_3$, 
then the boundary conditions would in fact be the natural ones.

\section{Numerics} \label{sec:implementation}
In this section we discuss the numerical minimization of the energy functional $\mathcal{F}$ defined in \eqref{eq:E}. 
Our approach is based on solving the optimality conditions for the minimizer $\vi{1},\vi{2}$
from \eqref{eq:optin}, \eqref{eq:pde2} with a fixed point iteration. 
We call the iterates of the fixed point iteration 
$\vci{i}{j}(\vx,t;k)$, where $k=1,2,\ldots,K$ denotes the iteration number and $K$ is the maximal number of 
iterations. We summarize all the iterates of the components of flow functions 
$\vci{i}{j}$ in a tensor of size $M \times N \times T \times K$.
In this section we use the notation as reported in Table \ref{tab:discrete}.
\begin{table}
\hspace*{-3em}
 \begin{tabular}{|S r|c|}
 \hline
 $f= f(r,s,t) \in \R^{M \times N \times T}$ & input sequence\\
\hline
 $\vi{i} = \vi{i}(r,s,t;k) \in \R^{M \times N \times T \times K \times 2}$ & discrete optical flow approximating the\\
 & continuous flow $\vi{i}$ at $(\frac{r}{M-1},\frac{s}{N-1},\frac{t}{T-1})$\\
\hline
 $\partial_k^h$  & finite difference approximation in \mbox{direction} $x_k$ \\
  \hline
 $\partial_t^h$ &  finite difference approximation in \mbox{direction} $t$ \\
 \hline 
    $\Delta_x=\frac{1}{M-1}$, $\Delta_y=\frac{1}{N-1}$ and $\Delta_t=\frac{1}{T-1}$ & Discretization\\ 
\hline
$ \vgcii{2}{j}(r,s,t;k) = \Delta_t \sum_{\tau=1}^t \vci{2}{j}(r,s,\tau;k)$,  $j=1,2$
& finite difference approximation of  $\widehat{u}(\cdot,t)$ \\
\hline
$ \vgci{2}{j}(r,s,t;k)= - \Delta_t \sum_{\tau = t}^T \vgcii{2}{j}(r,s,\tau;k)$ & finite difference approximation of $\widehat{\widehat{u}}(\cdot,t)$\\
\hline
 \end{tabular}
 \vspace{0.3cm}
\caption{Discrete Notation.}
\label{tab:discrete}
 \end{table}

For every tensor $H=(H(r,s,t)) \in \R^{M \times N \times T}$ (representing all frames of a complete movie) we define the 
discrete gradient 

\begin{equation*}
\begin{aligned}
\nabla_3^h H(r,s,t) &= (\partial_1^h H(r,s,t), \partial_2^h H(r,s,t), \partial_t^h H(r,s,t))^T\\ &\text{ for } 
(r,s,t) \in \{1,...,M\}\times \{1,...,N\} \times \{1,...,T\},
\end{aligned}
\end{equation*}
where 
\begin{equation}
\label{discreteSystem}
\begin{aligned}
\partial_1^h H(r,s,t)=&
\frac{H(r+1,s,t)-H(r-1,s,t)}{2\Delta_x}\quad \text{ if } 1<r<M, \\
\partial_2^h H(r,s,t)=&
\frac{H(r,s+1,t)-H(r,s-1,t)}{2\Delta_y}\quad \text{ if } 1<s<N, \\
\partial_t^h H(r,s,t) =&
\frac{H(r,s,t+1)-H(r,s,t-1)}{2\Delta_t}\quad \text{ if } 1<t<T.
\end{aligned}
\end{equation}
Let us notice that $r,s,t$ are discrete indices corresponding to the points in space 
$r \Delta_x, s \Delta_y$ and in time $t \Delta_t$ of a continuous function $H$.
Moreover, we define the discrete divergence, which is the adjoint of the discrete gradient:
Let $(H_1,H_2,H_3)^T(r,s,t)$, then 
\begin{equation}
\begin{aligned}
 \label{eq:disc_gradient}
 \nabla_3^h \cdot (H_1,H_2,H_3)^T &= \partial_1^h H_1 + \partial_2^h H_2+ \partial_t^h H_3 .
\end{aligned}
\end{equation}
The realization of the fixed point iteration for solving the discretized equations \eqref{eq:optin} and \eqref{eq:pde2} 
reads as follows:
\begin{itemize}
\item Initialization for $k=0$: we initialize to 0 
$\vi{1}(0)$, $\vi{2}(0) \in \R^{M \times N \times K \times 2}$. Note, we leave out all indices denoting space and time positions.
\item $k \to k+1$ and $k < K$: let $\nu'(k) := \nu'(|\nabla_3^h \vci{1}{1}(k)|^2 +|\nabla_3^h \vci{2}{1}(k)|^2)$, then 

\begin{equation}\hspace{-1cm}
\begin{alignedat}{2}
\label{eq:system_disc_1}
\frac{\vci{1}{1}(k+1)-\vci{1}{1}(k)}{\Delta_\tau}& = 
\nabla_3^h \cdot \bigl(&& \nu'(k) \nabla_3^h{\vci{1}{1}(k)} \bigr)\\
&- \frac{\partial_1^h f}{\alpha^{(1)}} \biggl[&& \partial_1^h f \left(\vci{1}{1}(k+1)+\vci{2}{1}(k) \right) \biggr. \\
  & &&+ \biggl. \partial_2^h f \left(\vci{1}{2}(k)+\vci{2}{2}(k)\right) + \partial_t^h f \biggr],
\end{alignedat} 
\end{equation}
\begin{equation}
\label{eq:system_disc_2}
\begin{alignedat}{2}
\frac{\vci{1}{2}(k+1)-\vci{1}{2}(k)}{\Delta_\tau}
 &= \nabla_3^h \cdot (&& \nu'(k) \nabla_3^h{\vci{1}{2}(k)})\\
  & -\frac{\partial_2^h f}{\alpha^{(1)}} \biggl[&& \partial_1^h f \left(\vci{1}{1}(k+1)+\vci{2}{1}(k) \right) \biggr. \\
&  && + \biggl. \partial_2^h f \left(\vci{1}{2}(k+1)+\vci{2}{2}(k)\right)+\partial_t^h f \biggr],
\end{alignedat} 
\end{equation}
\begin{equation}
\label{eq:system_disc_3}
\begin{alignedat}{2}
\frac{\vci{2}{1}(k+1)-\vci{2}{1}(k)}{\Delta_\tau}
 &= \vgci{2}{1}(k)&&\\
 &-\frac{\partial_1^h f}{\alpha^{(2)}} \biggl[[&& \partial_1^h f \left(\vci{1}{1}(k+1)+\vci{2}{1}(k+1)\right) \biggr.\\
& && \biggl. + \partial_2^h f \left( \vci{1}{2}(k+1)+\vci{2}{2}(k)\right)+\partial_t^h f \biggr],
 \end{alignedat} 
\end{equation}

and

\begin{equation}
\label{eq:system_disc_4}
\begin{alignedat}{2}
\frac{\vci{2}{2}(k+1)-\vci{2}{2}(k)}{\Delta_\tau} &=  \vgci{2}{2}(k)&&\\
&-\frac{\partial_2^h f}{\alpha^{(2)}} \biggl[&& \partial_1^h f \left(\vci{1}{1}(k+1)+\vci{2}{1}(k+1)\right) \biggr.\\ 
& && \biggl. + \partial_2^h f \left( \vci{1}{2}(k+1)+\vci{2}{2}(k+1) \right)+\partial_t^h f\biggr],
\end{alignedat} 
\end{equation}

where $\Delta_\tau$ is the step size iteration parameter, which has been set to $10^{-4}$. 
In order to improve stability of the algorithm and ensure convergence we use a semi implicit iteration scheme 
as proposed in \cite{WeiSchn01b}. Indeed, let us notice that in \eqref{eq:system_disc_1}, \eqref{eq:system_disc_2},\eqref{eq:system_disc_3}, \eqref{eq:system_disc_4} one of the components of the two flow fields 
$\vi{1}$, $\vi{2}$ on the right hand side is evaluated for iteration index $k+1$.  
The system \eqref{eq:system_disc_1}, \eqref{eq:system_disc_2}, \eqref{eq:system_disc_3}, \eqref{eq:system_disc_4} 
can be solved efficiently using the special structure of the matrix equation (see \cite{WeiSchn01a,WeiSchn01b}). 
The matrix equation then is a sparse block tridiagonal matrix.This type of matrices  allows for efficient estimation of a solution through decomposition and parallelization methods \cite{LeeWri10}.
\end{itemize}
The iterations are stopped when the Euclidean norm of the relative error
$$\frac{|\vci{i}{j}(\cdot,k)-\vci{i}{j}(\cdot,k+1)|}{|\vci{i}{j}(\cdot,k)|}
,\qquad j=1,2$$ 
drops below the precision tolerance value $tol=0.05$ for at least one of the component of $\vu^{(1)}$ and one of $\vu^{(2)}$.
The typical number of iterations is much below $100$.

\section{Experiments} \label{sec:experiments}
In this section we present numerical experiments to demonstrate the potential of the proposed optical flow decomposition model.
In the first two experiments we use for visualization of the computed flow fields the standard flow color coding \cite{BakSchaLewRotBla11}. 
The flow vectors are represented in color space via the color wheel illustrated in Figure \ref{fig:colorWheel}.
For the third and fourth experiment we use a black and white visualization technique: Black means that there is 
no flow present and the gray-shade is proportional to the flow magnitude. In order to compare the optical flow computations 
quantitatively the intensity values of $f$ have been scaled to the range $[0,1]$.
\begin{figure}[htp]
 \centering
 \includegraphics[width=0.3\textwidth]{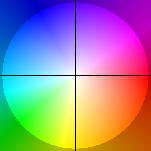}\quad
\caption{ Color Wheel. }
\label{fig:colorWheel}
\end{figure}
The used parameters are reported for each experiment separately, except for the discretization parameters 
$\Delta_x,\Delta_y,\Delta_t$, defined in Section \ref{sec:implementation}. 
In this work we consider the following four dynamic image sequences:
\begin{itemize}
\item The first experiment is performed with the video sequence from \cite{McCNovCraGal01} 
      (available at \url{http://of-eval.sourceforge.net/}) which 
      consists of forty-six frames showing a rotating sphere with some overlaid patterns.
       The analytical results from Appendix \ref{app:proof} in 1D suggest that the intensity of the 
       $\vi{2}$ component increases monotonically with increasing rotational frequency over time. 
       We verify this hypothesis numerically, now in higher dimensions. We simulate in particular two, 
       four and eight times the original motion frequency;: In order to do so, we duplicate the sequence 
       periodically, however consider it to be recorded in the same time interval $(0,1)$.  
       The flow visualized in Figure \ref{fig:u2DF} is the one between the 16th and the 17th frame of 
       every sequence. We study the behavior of the sphere at different motion frequencies with the same 
       parameter setting $\alpha^{(1)}=1$, $\alpha^{(2)}=\frac{1}{4}$. The numerical results confirm the 
       $1D$ observation that for high rotational movement $\vi{2}$ is dominant 
       (cf. Figures \ref{fig:u2DF}) and $\vi{1}$ is always $20\%$ of $\vi{2}$; in particular 
       $\vi{1}$ and $\vi{2}$ cannot be completely separated.
        \begin{figure}[htp]
        \begin{center}
        \includegraphics[width=0.3\textwidth]{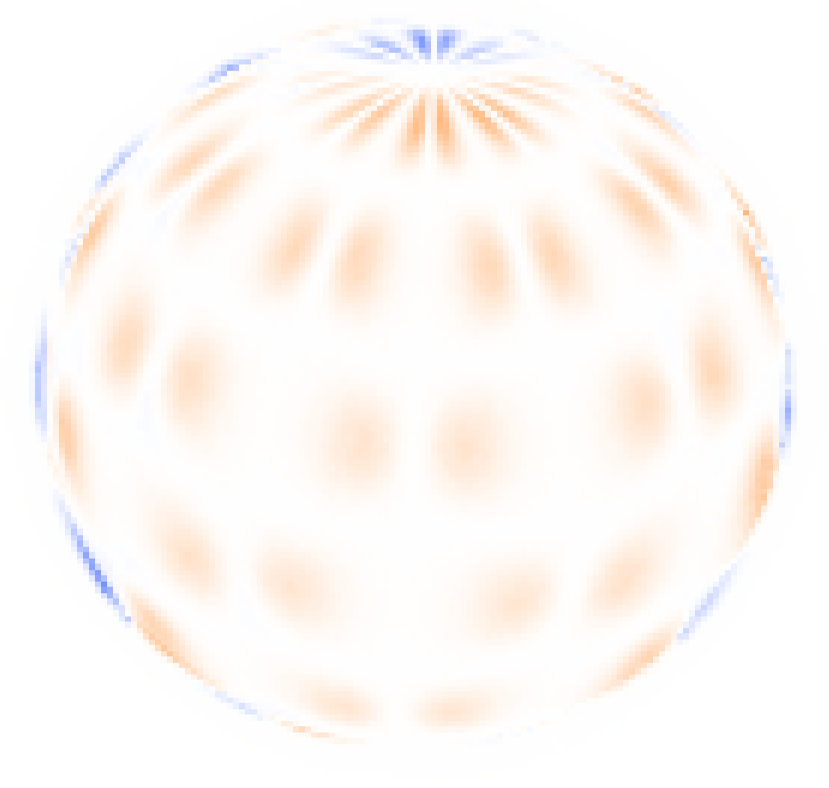} \quad 
        \includegraphics[width=0.3\textwidth]{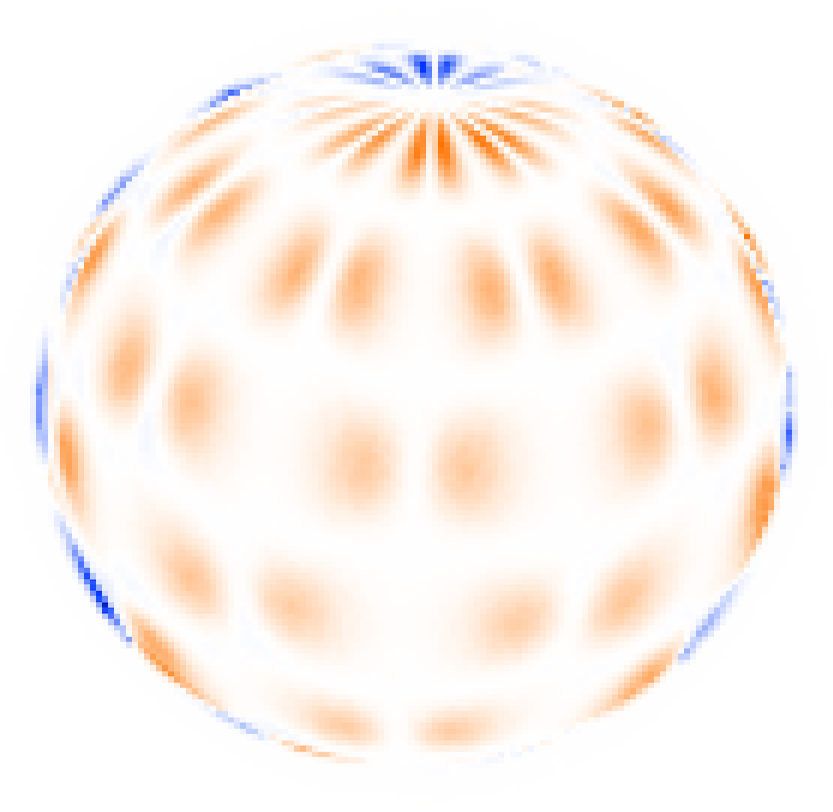} \quad
        \includegraphics[width=0.3\textwidth]{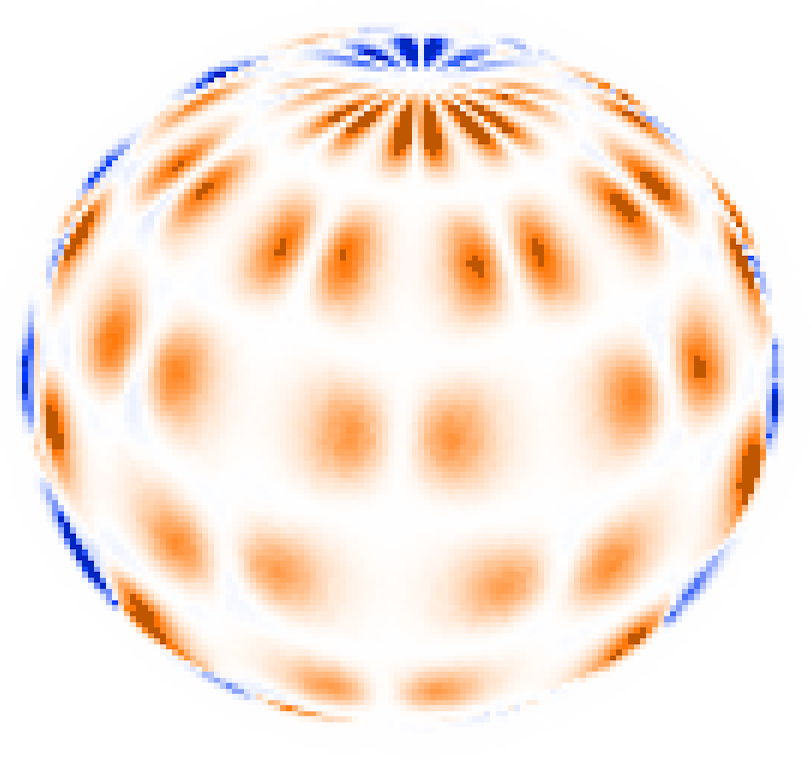}
        \label{fig:u2DF}
        \end{center}        
        \caption{$\vi{2}$ at different frequencies of rotations: $2$, $4$ and $8 \times$ faster than the original motion frequency. 
                 $\alpha^{(1)}=1$, $\alpha^{(2)}=\frac{1}{4}$. The intensity of $\vi{2}$ increases when the
                 frequency of rotation is increased.}
        \end{figure}
        
\item The second experiment concerns the decomposition of the motion in a dynamic image sequence showing a projection of a 
      cube moving over an oscillating background. The movie consists of sixty frames and can be viewed on the web-page 
      \url{http://www.csc.univie.ac.at/index.php?page=visualattention}.\linebreak 
      The background is oscillating in diagonal direction,  from the bottom left to the top right, with a periodicity of 
      four frames. In each frame the oscillation has a rate of $5\%$ of the frame size. The flow visualized in 
      Figure \ref{fig:boxBO} is the one between the 20th and the 21st frame of the sequence. 
      
      Applying the proposed method with a parameter setting $\alpha^{(1)}=10^3$, $\alpha^{(2)}=10^3$, $\Delta_\tau=10^{-5}$, 
      and precision tolerance $tol=0.001$, we notice that the background movement appears almost solely in $\vi{2}$ and the 
      global movement of the cube appears in $\vi{1}$. In Figure \ref{fig:boxBO} we represent only flow vectors with magnitude 
      larger than $0.05$ and omit in $\vi{2}$ the part in common with $\vi{1}$ for better visibility.
 	\begin{figure}[htp] \label{fig:boxBO} 
	\begin{center}
	\includegraphics[width=0.45\textwidth]{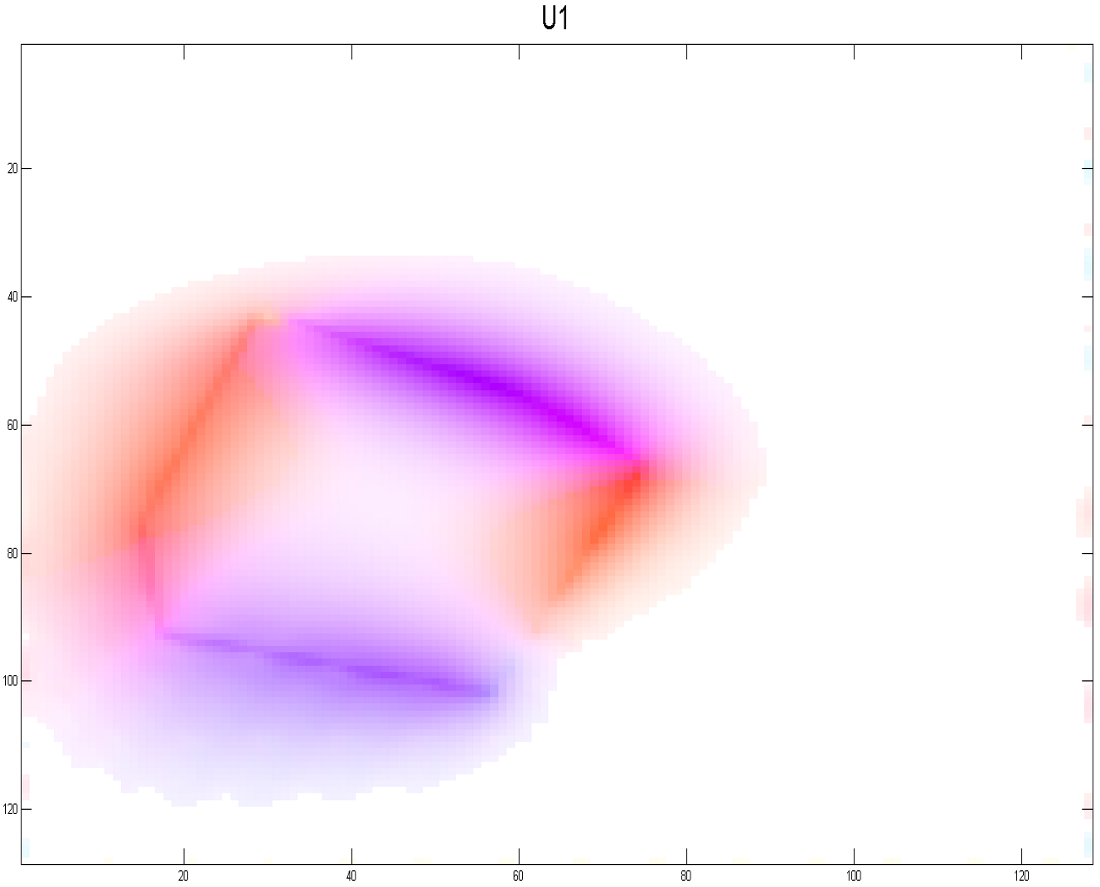} \quad 
        \includegraphics[width=0.45\textwidth]{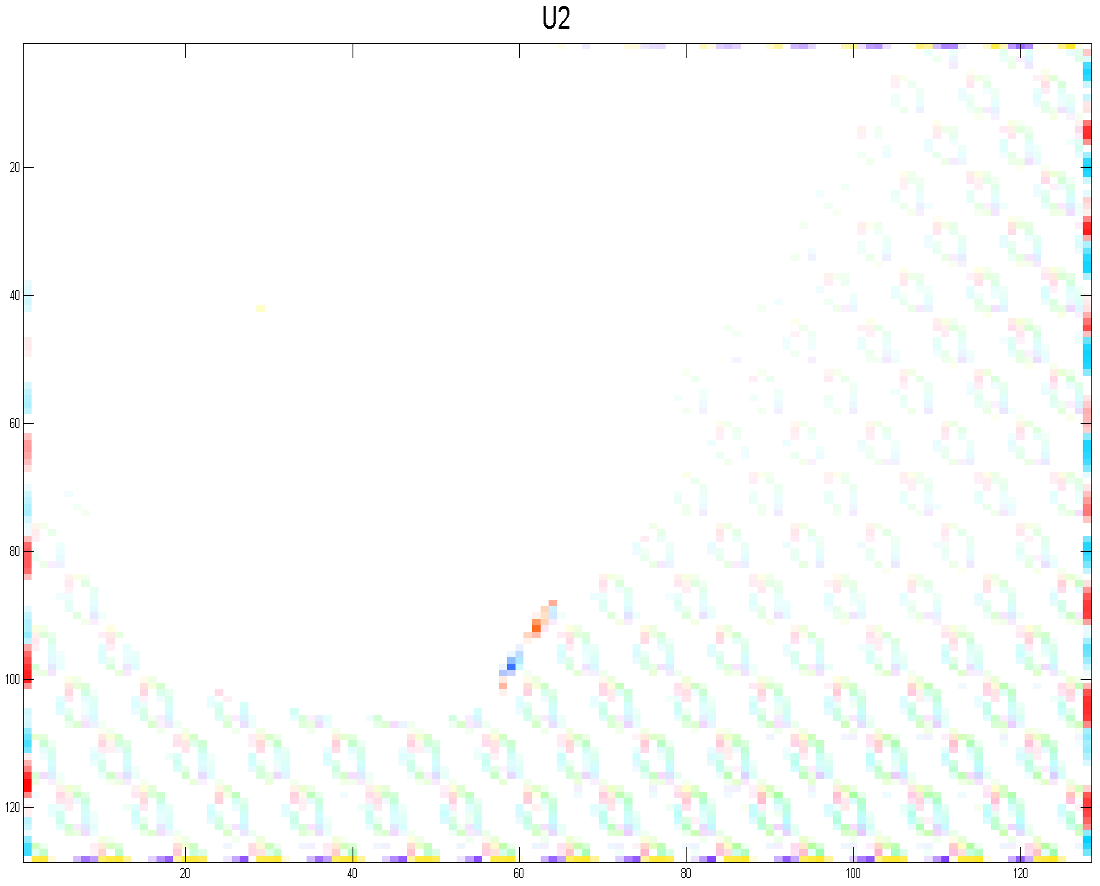}  
        \end{center}
        \caption{The dynamic sequence consists of the smooth (translation like) motion of a cube and an oscillating background.
                 The oscillation has a periodicity of four frames and takes place along the diagonal direction from the bottom 
                 left to the top right, moving at a rate of 5\% of the frame size in each frame. The proposed model decomposes 
                 the motion, obtaining the global movement of the cube in $\vi{1}$ (left) and the background movement exclusively in 
                 $\vi{2}$ (right).}
	\end{figure}	

\item In the third experiment the original movie consists of thirty-two frames and can be viewed together with the 
      decomposition result on the web-page 
      \url{http://www.csc.univie.ac.at/index.php?page=visualattention}.
      The flow is decomposed into two components.
      The first part shows the movement of a Ferris wheel and people walking. 
      The second part shows blinking lights and the reflection of the wheel.
      The flow visualized in Figure \ref{fig:giostra} is the one between the 4th and the 5th frame of the sequence with a parameter 
      setting $\alpha^{(1)}=1$, $\alpha^{(2)}=\frac{1}{4}$. In order to improve the visibility we represent only flow
      vectors with magnitude larger than $0.18$ and we omit for $\vi{2}$ the part in common with $\vi{1}$.
      \begin{figure}[tp]
	\begin{center}
         \includegraphics[width=0.45\textwidth]{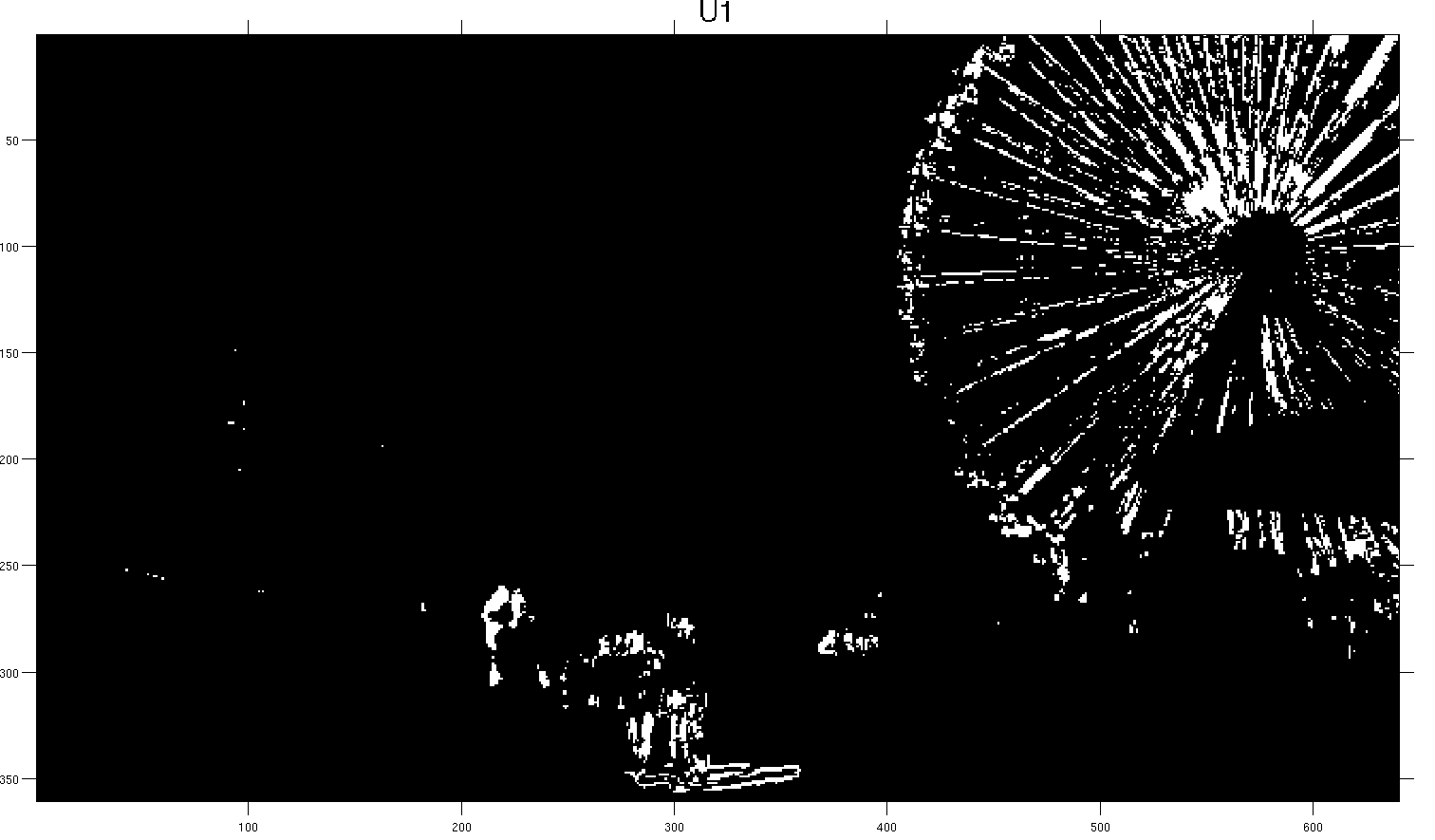} 
         \includegraphics[width=0.45\textwidth]{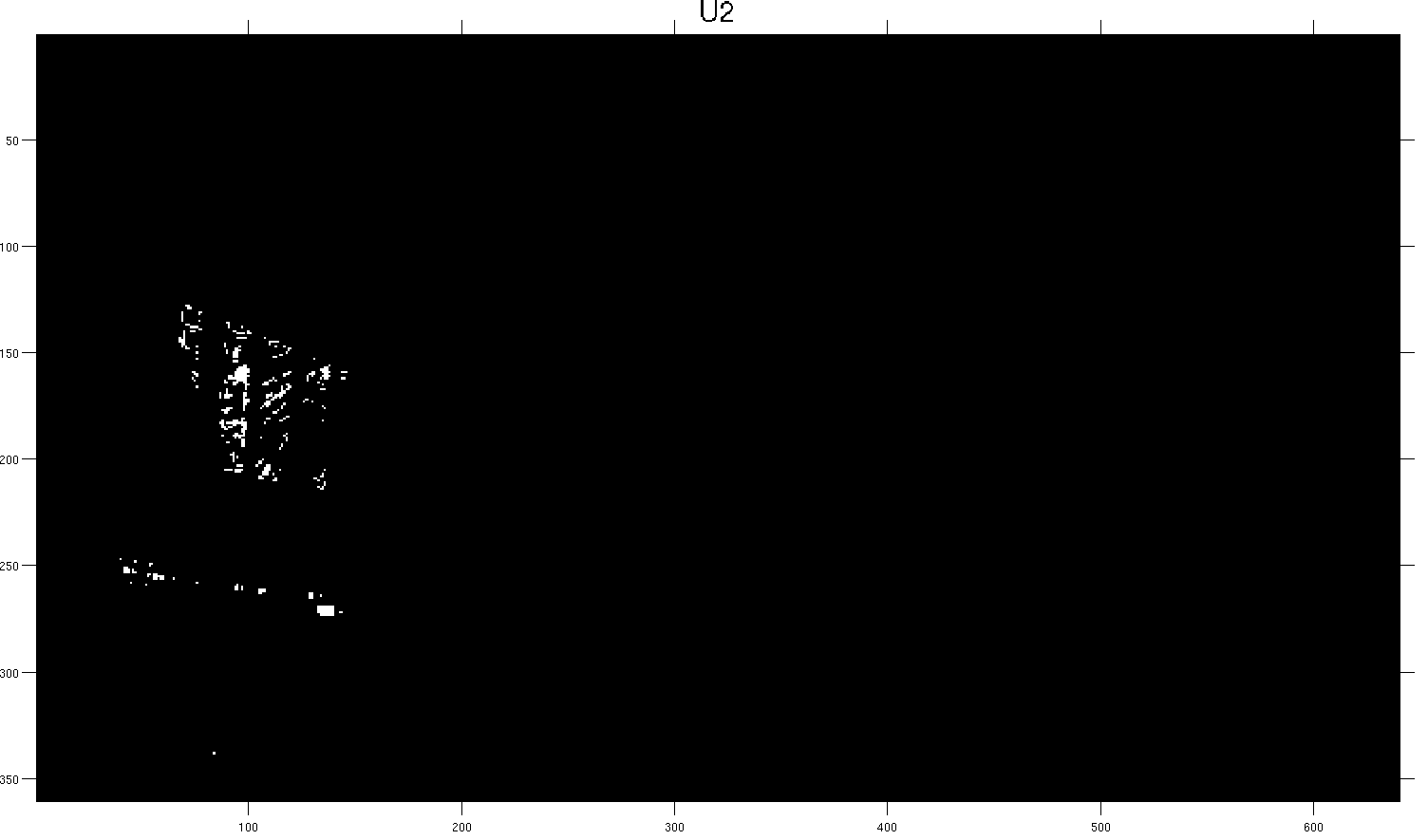} \\
	 \includegraphics[width=0.45\textwidth]{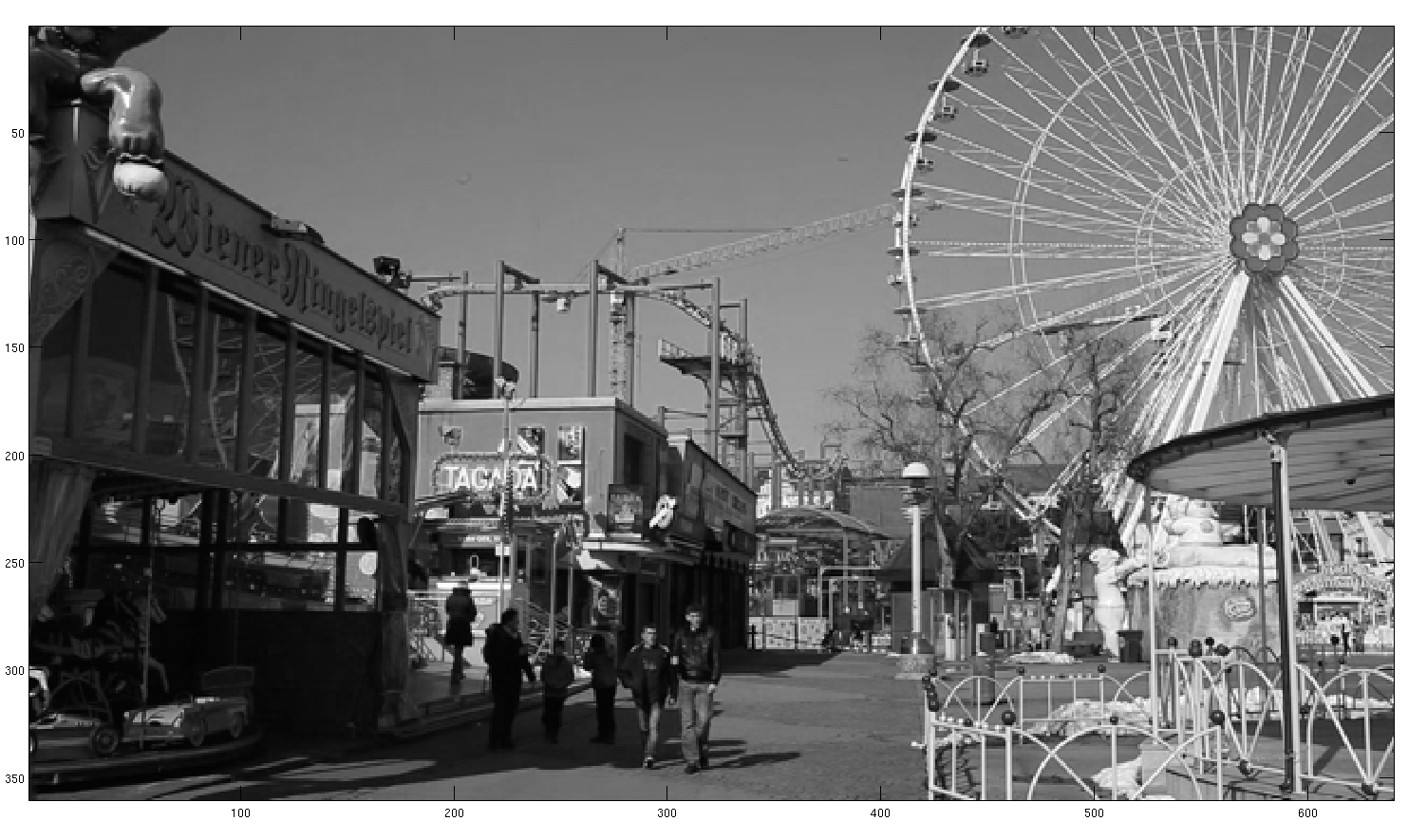}   
        \end{center}
         \caption{$\vi{1}$: Movement of a Ferris wheel and people walking in the foreground (top left). 
         $\vi{2}$ consists of blinking lights and the reflections of the wheel (top right). The third image (bottom) is 
         a reference frame.} \label{fig:giostra}
	\end{figure} 
\item The fourth example is flickering.
      In a standard flickering experiment, the difference in human attention is investigated by inclusion of blank 
      images in a repetitive image sequence. Although, in general, these blank images are not deliberately recognized, 
      they change the awareness of the test persons. J. K. O'Regan \cite{Ore07} states that
      ``{\bf Change blindness is a phenomenon in which a very large change in a picture will not be seen by a viewer, 
      if the change is accompanied by a visual disturbance that prevents attention from going to the change location}''. 
      They test data from \url{http://nivea.psycho.univ-paris5.fr}, was used for our simulations.       
      The proposed optical flow decomposition is able to detect regions, which also humans can recognize, but standard 
      optical flow algorithms fail to: To show this,  
      the input sequence is composed by four frames consisting of Frame $1$, a blank image, Frame $2$ and again an 
      identical blank image (see Figure \ref{fig:fliBox} (top)). 
      This sequence is then aligned periodically to a movie. We interpret the movie as a linear interpolation between 
      the frames. We test and compare Horn-Schunck, Weickert-Schn\"orr and the proposed algorithm. 
 \begin{figure}[htp]
 \begin{center}
\includegraphics[scale=0.24]{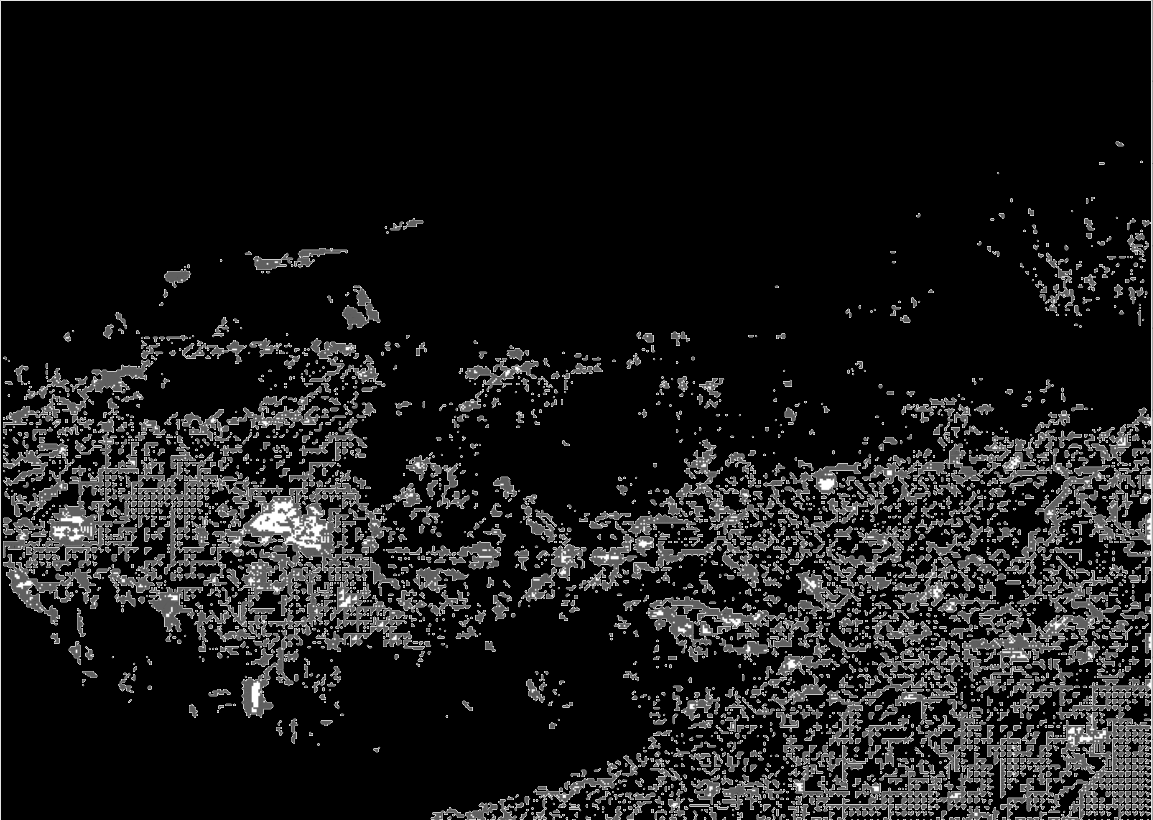}
\end{center}
\caption{Result with Horn-Schunck.}
\label{fig:resultFlicker}
\end{figure}

We set the smoothness parameter $\alpha^{(1)}$ 
to a value of one for all the methods.  Moreover, for our approach we set $\alpha^{(2)}=1$.
For Horn-Schunck we visualize the flow field in Figure \ref{fig:resultFlicker}. This flow is the one between the blank frame and the 
slightly changed frame,  which exceeds a threshold of $3.9$.
The results obtained by applying Weickert-Schn\"orr and the $\vi{1}$ field of our approach, respectively, are small in magnitude. 
Therefore, we do not visualize them.  
This behavior is coherent with the motivation of the Weickert-Schn\"orr method to produce an optical flow that is less 
sensitive to variations over space and time.
Finally, we visualize in Figure \ref{fig:fliBox} (down right) 
the $\vi{2}$ flow field for the proposed approach. For the visualization we omit all
vectors with magnitude lower than $0.18$.
In order to make transparent the result, we show in Figure \ref{fig:fliBox} (down left) the difference between 
the two frames of the sequence containing information (see Figure \ref{fig:fliBox} (top)).
\begin{figure}[htp]
\begin{center}
\includegraphics[scale=0.2]{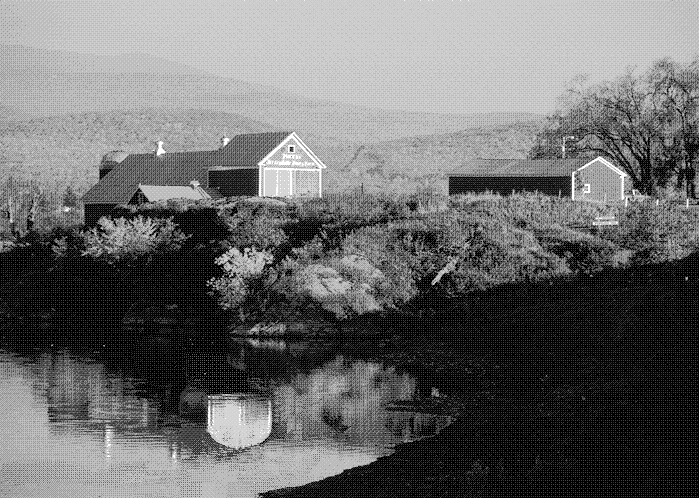}\quad
\includegraphics[scale=0.20]{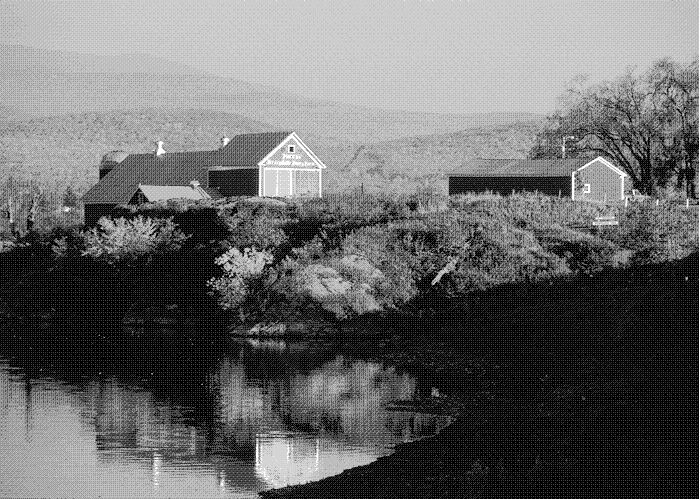}\\
\vspace{0.1cm}
\includegraphics[scale=0.418]{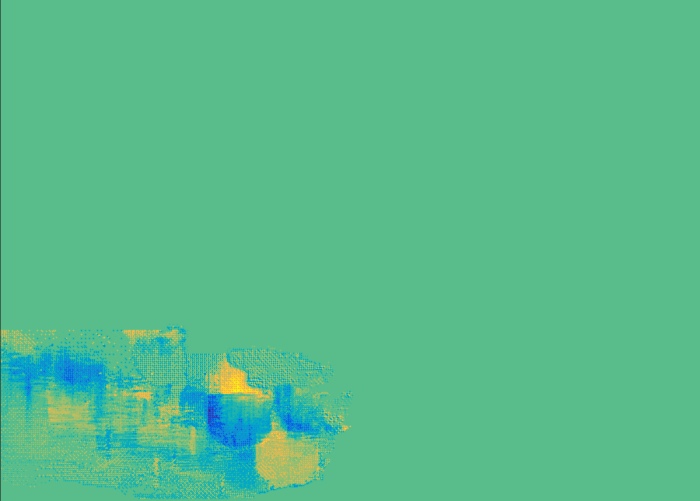}\quad
\includegraphics[scale=0.201]{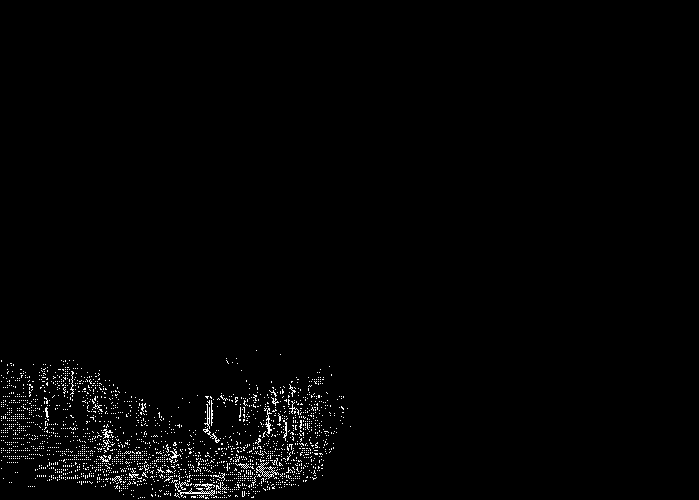}
\end{center}
\caption{The two frames of the flickering sequence containing information (top), the difference between these two 
frames (down left), and the $\vi{2}$ flow field resulting from the proposed approach (down right). 
As predicted in Section \ref{sec:anEx} and Appendix \ref{app:proof} the $\vi{1}$ component is negligible, 
instead $\vi{2}$ detects the change of intensity across the blank sheet.
}
\label{fig:fliBox}
\end{figure}
In this experiment, we notice that the $\vi{1}$ component is negligible, instead $\vi{2}$ detects the areas affected 
by change of intensities (see Figure \ref{fig:fliBox} (down right)).
\end{itemize}

\subsection*{Additional Information}
In the following, we show the capacity of our model
to extract different information compared to standard optical flow algorithms.
The current literature focuses on average angular and endpoint error \cite{BakSchaLewRotBla11} 
in order to compare optical flow algorithms.
Our model extracts information, that is neglected by standard algorithms. 
Such difference can be shown through a quantitative comparison of models.
For this purpose, we use well-known test sequences from the Middlebury database \url{http://vision.middlebury.edu/flow/}, 
and evaluate the residual of the {\bf optical flow constraint}. 

In order to understand how much information our method is capable to extract
from an entire dynamic sequence, we calculate the residual squared over space and time:
$\mathcal{E}(\vi{1},\vi{2})$ as in \eqref{eq:E} and compare it with the squared residual over space and time of 
the Weickert-Schn\"orr method \cite{WeiSchn01a,WeiSchn01b}. 
We use the parameter settings $\alpha^{(1)}=100$ ($\alpha = \alpha^{(1)}$ in Weickert-Schn\"orr)
and  $\alpha^{(2)}=\frac{1}{4}$, tolerance $tol=0.01$, in order to have a good comparison of the two methods.
Again the residual is smaller for the proposed method as shown in Table \ref{tab:tabellaRG}.
\begin{table}[bt]
\centering
\begin{tabular}{|l|c|c|r|}
\hline
 & Weickert-Schn\"orr & Proposed model \\
\hline
Hamburg Taxi & 1374.9 & 1021 \\
\hline
RubberWhale & 4459.7 & 3046.8\\
\hline
Hydrangea & 8533.3 & 7647.2 \\
\hline
DogDance  & 9995.4 &  8217.6 \\
\hline
Walking  & 8077.5 &  5944.3 \\
\hline
\end{tabular}
\vspace*{0.1cm}
\caption{Comparison of squared residuals over space and
time $\mathcal{E}$ between Weickert-Schn\"orr and the proposed method.}
\label{tab:tabellaRG}
\end{table}

The smaller value of the residual are due to the fact that we are calculating a minimizer from a larger space of optical flow 
components. 
However, with this approach we cannot observe oversmoothing of the flow.

\section{Conclusion}\label{sec:conclusion}
We present a new optical flow model for decomposing the flow in spatial and temporal components of different scales.
A main ingredient of our work is a new variational formulation of the optical flow equation. Finally, applications (some of them 
from psychological experiments) are considered and analyzed analytically and computationally.

\section*{Acknowledgment}
This work is carried out within the project 
{\bf Modeling Visual Attention as a Key Factor in Visual Recognition and Quality of Experience} funded by the 
Wiener Wissenschafts und Technologie Fonds - WWTF.
OS is also supported by the Austrian Science Fund - FWF, Project S11704 with the national research network, 
NFN, Geometry and Simulation.
The authors would like to thank U.~Ansorge, C.~Valuch, S.~Buchinger, C. Kirisits, P. Elbau, L. Lang, A. Beigl, T. Widlak 
for interesting discussions on the optical flow and J.~ A.~ Iglesias for discussions and the creation of
the cube sequence.

\begin{appendix}
\section{Optical flow decomposition in 1D}\label{app:proof}
In order to make transparent the features of our decomposition we study exemplary 
the 1D case again. From regularization theory (see e.g. \cite{SchGraGroHalLen09}) 
we know that the minimizers $(\vie{1}_{\vec{\alpha}}$, $\vie{2}_{\vec{\alpha}})$, 
for $\vec{\alpha} = (\alpha^{(1)},\alpha^{(2)}) \to 0$, are converging to 
a solution of the optical flow equation which minimizes 
\begin{equation*}
 \mathcal{R} = \Reg{1}+\alpha \Reg{2} \text{ for } \alpha = \lim_{\vec{\alpha} \to 0} \frac{\alpha^{(2)}}{\alpha^{(1)}} > 0\;.
\end{equation*}
Such a solution is called $\mathcal{R}$ minimizing solution.
Note that by the 1D simplification the modules $u^{(i)}$, $i=1,2$, are real valued functions.

We calculate the decomposition for the optical flow equation \eqref{eq:ofe_1}, for 
the specific test data \eqref{eq:specific}. The regularized solutions approximate the $\mathcal{R}$ minimizing solution, 
and thus these calculations can be viewed representative for the properties of the minimizer of the regularization method.
For these particular kind of data the solution of the optical flow equation is given by :
\begin{equation}
\label{eq:of_1d}
u(x,t) = - \frac{\tilde{f}(x)}{\partial_x \tilde{f}(x)} \frac{\partial_t g(t)}{g(t)} = 
- \frac{\partial_t (\log g)(t)}{\partial_x (\log \tilde{f})(x)} \;.
\end{equation}
Let us assume that $(\log g)(t)-(\log g)(0)$ can be expanded into a Fourier $\sin$-series:
\begin{equation} 
\label{eq:Hseries}
(\log g)(t)-(\log g)(0)= \int_0^t \partial_t (\log g)(\tau)\,\d \tau =\sum_{n=1}^\infty \hat{g}_n \sin(n \pi t) .
\end{equation}
Moreover, we assume that $1/\partial_x (\log \tilde{f})(x)$ can be expanded in a $\cos$-series:
\begin{equation}
 \label{eq:gseries}
 \frac{1}{\partial_x (\log \tilde{f})(x)} = \sum_{m=0}^\infty f_m \cos(m \pi x) .
\end{equation}
Then
\begin{equation}
\label{eq:Gf}
 - \frac{(\log g)(t)-(\log g)(0)}{\partial_x (\log \tilde{f})(x)} = \widehat{u}(x,t) =
    \widehat{u}^{(1)}(x,t) + \widehat{u}^{(2)}(x,t).
\end{equation}
Inserting this identity in the regularization functional $
\mathcal{R} (\vie{1},\vie{2})=\int 
(\partial_x \vie{1})^2+(\partial_t \vie{1})^2 
+\alpha\left(\Vi{2}\right)^2 \d x \d t\,,$
we remove the $\vie{2}$ dependence, and we get
\begin{equation*}
\mathcal{R} (\Vi{1})
:= \int_{(0,1) \times (0,1)}\!\!\!\!\!\!\! (\partial_{xt} \Vi{1})^2+(\partial_{tt} \Vi{1})^2 
   + \alpha \left(\frac{(\log g)(t)-(\log g)(0)}{\partial_x (\log \tilde{f})(x)} + \Vi{1} \right)^2 \d x \d t,
\end{equation*}
where we enforce the following boundary conditions on $\Vi{1}$: From the definition of $\Vi{1}$ it follows that 
$\Vi{1}(x,0)=0.$
Secondly, we enforce $\Vi{1}(x,1)= 0$.
In fact, the assumption is reasonable, because when the series $\sum_{n=0}^\infty \hat{g}_n$ is absolutely 
convergent, $(\log g)(t)-(\log g)(0)=0$ (cf. \eqref{eq:Hseries}), which implies that $\Vi{1}(x,1)+\Vi{2}(x,1)=0$ 
(cf. \eqref{eq:Hseries}).

By substituting the relation between $\Vi{2}$ and $\Vi{1}$ we reduce the constraint optimization problem to an 
unconstrained optimization problem for $\Vi{1}$, and the minimizer solves the partial differential equation
\begin{equation*}
 \partial_{ttxx} \Vi{1} +\partial_{tttt} \Vi{1} + \alpha 
\left(\frac{(\log g)(t)-(\log g)(0)}{\partial_x (\log \tilde{f})(x)} + \Vi{1} \right) =0 \text{ in } (0,1) \times (0,1),
\end{equation*}
together with the boundary conditions:
\begin{equation}
\label{eq:u1}
 \begin{aligned}
    \partial_{tt} \Vi{1} = \Vi{1} &= 0 \text{ on } (0,1) \times \set{0,1},\\
  \partial_x \partial_{tt} \Vi{1} &= 0 \text{ on } \set{0,1} \times (0,1). 
 \end{aligned}
\end{equation}

The boundary conditions $\partial_{tt} \Vi{1} = = 0 \text{ on } (0,1) \times \set{0,1},$ and 
$\partial_x \partial_{tt} \Vi{1} = 0 \text{ on } \set{0,1} \times (0,1)$ appear as natural boundary conditions, when weak solutions 
are considered.

Now, we substitute 
$\hat{w} := \partial_{tt} \Vi{1},$
and we get the following system of equations
\begin{equation}
\label{eq:w}
\begin{aligned}
  \partial_{xx} \hat{w} +\partial_{tt} \hat{w} &= - \alpha \left(\frac{(\log g)(t)-(\log g)(0)}{\partial_x (\log \tilde{f})(x)}+\Vi{1} \right) \text{ in } (0,1) \times (0,1),\\
  \hat{w} &= 0 \text{ on } (0,1) \times \set{0,1},\\
  \partial_x \hat{w} &= 0 \text{ on } \set{0,1} \times (0,1).
\end{aligned}
\end{equation}
and 
\begin{equation*} 
\begin{aligned}
\Vi{1}(x,t) &= \int_0^t \int_0^\tau \hat{w}(x,\hat{\tau})\d \hat{\tau} \d \tau - t \int_0^1 \int_0^\tau \hat{w}(x,\hat{\tau})\d \hat{\tau} \d \tau .
\end{aligned}
\end{equation*}
$\hat{w}$ can be expanded as follows:
\begin{equation*}
 \hat{w}(x,t) = \sum_{m,n=0}^\infty \hat{w}_{mn} \cos (m \pi x) \sin(n \pi t),
\end{equation*}
and we expand $\Vi{1}$ in an analogous manner: 
\begin{equation*}
  \Vi{1}(x,t) = \sum_{m,n=0}^\infty \Vi{1}_{mn} \cos (m \pi x) \sin(n \pi t)\;,
\end{equation*}
such that 
\begin{equation}
\label{eq:system1b}
\hat{w}_{mn} = - n^2 \pi^2 \Vi{1}_{mn}, \quad \forall\, m,n \in \N_0.
\end{equation} 

Thus it follows from \eqref{eq:w} that 
\begin{equation}
\begin{aligned}
\label{eq:system1}
\hat{w}_{mn} (m^2 + n^2)\pi^2  =  \alpha\left(\Vi{1}_{mn} + f_m \hat{g}_n\right), \quad \forall\, m,n \in \N_0.
\end{aligned}
\end{equation}

\eqref{eq:system1b} and \eqref{eq:system1} imply that 
\begin{equation}
\begin{aligned}
\label{eq:system2}
\Vi{1}_{mn} = - \frac{\alpha}{\alpha + \pi^4 (m^2 + n^2) n^2} f_m \hat{g}_n, \quad \forall\, m,n \in \N_0.
\end{aligned}
\end{equation}

Now, consider a specific test example $g(t) = \exp\set{\frac{\sin(n_0\pi t)}{n_0 \pi}}$ for some $n_0 \in \N$.
Then, from \eqref{eq:of_1d} it follows that 
$u(x,t) = - \frac{\cos(n_0 \pi t)}{\partial_x (\log \tilde{f})(x)},$
and correspondingly we have 
\begin{equation*}
(\log g)(t)-(\log g)(0) = \frac{\sin (n_0 \pi t)}{n_0 \pi} = \sum_{n=1}^\infty  \frac{\delta_{n n_0}}{n_0 \pi}
\sin (n \pi t).
\end{equation*}
In this case it follows from \eqref{eq:system2} that 
\begin{equation*}
 \Vi{1}_{mn}  = -\frac{\alpha}{\alpha + \pi^4 (m^2 + n_0^2)n_0^2} \frac{\delta_{n n_0}}{n_0 \pi} f_m .
\end{equation*}
In the case of flickering data $f$, $\vie{2}$ is pronounced (if $n_0$ is large $\Vi{1}_{mn} \approx 0)$ 
while in the quasi-static case $\vie{1}$ is dominant.
Moreover, we also see that spatial components belonging to Fourier-$\cos$ coefficients with large $m$ 
are more pronounced in the $\vie{2}$ component, and the spatial and temporal coefficients always appear in both components.
In particular this means that a threshold has to be set, to assign them to the first or second module.
\end{appendix}


\end{document}